\documentclass[lettersize,journal]{IEEEtran}

\IEEEoverridecommandlockouts                              

\usepackage[utf8]{inputenc}
\usepackage{amsmath}
\usepackage{amssymb}
\usepackage{amsfonts}
\usepackage{bm}

\usepackage{amsthm}
\usepackage{overpic}
\usepackage{lipsum}
\usepackage{enumitem}
\usepackage{cases}
\usepackage{subcaption}
\usepackage{graphicx}
\usepackage{ulem}
\usepackage{xcolor}
\usepackage{cite}

\usepackage{algorithm,algorithmic}
\usepackage{mathbbol}
\usepackage{comment}

\usepackage[colorlinks = true,
            linkcolor = red,
            urlcolor  = magenta,
            citecolor = green,
            anchorcolor = blue]{hyperref}

\DeclareMathOperator*{\argmax}{arg\,max}

{
    \theoremstyle{plain}
    \newtheorem{assumption}{Assumption}
}

\DeclareCaptionLabelFormat{lc}{\MakeLowercase{#1}~#2}

\newtheorem{theorem}{Theorem}
\newtheorem{corollary}{Corollary}

\newtheorem{definition}{Definition}
\newtheorem{proposition}{Proposition}
\newtheorem{example}{Example}[]
\newtheorem{problem}{Problem}
\newcommand{\edit}[1]{\textcolor{black}{#1}}

\makeatletter
\def\maketag@@@#1{\hbox{\m@th\normalfont\normalsize#1}}
\makeatother

\title{\LARGE \bf
Cooperative Risk-Aware Exploration in Heterogeneous Multi-Robot Systems Using Algorithmic Altruism
}

\author{Brooks A. Butler$^1$, Jair Cert\'{o}rio$^2$, Jo\~{a}o P. Hespanha$^3$, and Magnus Egerstedt$^4$
\thanks{$^1$Brooks A. Butler is with the School of Electrical and Computer Engineering at Oklahoma State University, Stillwater, OK, 74074, USA (Email: {\tt\small brooks.butler@okstate.edu}).
}%
\thanks{$^2$Jair Certório is with the Systems Engineering Division, Instituto Tecnológico de Aeronáutica, Fortaleza, CE, 60415-513, Brazil (Email: {\tt\small certorio@ita.br}).}%
\thanks{$^3$Jo\~{a}o P. Hespanha is with the Electrical and Computer
Engineering Department, University of California, Santa Barbara, CA, 93106,
USA (Email: {\tt\small hespanha@ece.ucsb.edu}).}%
\thanks{$^4$Magnus Egerstedt is with the Department of Computer Science at the University of North Carolina, Chapel Hill, NC, 27599, USA (Email: {\tt\small magnus@unc.edu).}
}%
\thanks{This research was supported in part by an appointment to the Intelligence Community Postdoctoral Research Fellowship Program at the University of California, Irvine, administered by the Oak Ridge Institute for Science and Education (ORISE) through an interagency agreement between the U.S. Department of Energy and the Office of the Director of National Intelligence (ODNI).}%
}

\begin{document}


\maketitle
\thispagestyle{empty}
\pagestyle{empty}

\begin{abstract}
Multi-robot systems are well-positioned for exploration in hazardous environments, but effective deployment requires deciding not only where robots should gather information, but also how risk should be distributed across heterogeneous team members. This paper develops a game-theoretic framework for cooperative risk-aware exploration based on ecologically inspired altruistic behavior. Each robot selects a finite-horizon trajectory to maximize information gain while penalizing redundant exploration and expected hazard exposure. Heterogeneity is introduced through agent-specific value parameters for encoding altruistic coupling, which is modeled through relatedness weights inspired by Hamilton's rule. We introduce a game-theoretic structure for trajectory planning that defines a Social Nash Equilibrium, which modifies the utility of agent actions according to agent relatedness. This utility shaping causes agents to internalize the effect of their trajectory choices on teammates, encouraging lower-valued robots to accept risk when doing so benefits higher-valued agents and improves team performance. We define an exploration utility for agents that rewards area coverage and uncertainty reduction, while also penalizing redundancy and risk, enabling projected gradient-based waypoint optimization in a receding-horizon planner. Simulations show that altruistic planning reduces redundant exploration, improves inter-robot separation, and reallocates risk according to agent value while maintaining comparable map coverage. We further demonstrate the approach in hardware experiments, where planned waypoints are tracked by wheeled robots using single-integrator controllers and barrier certificates.
\end{abstract}

\section{Introduction}

Multi-robot systems are frequently proposed for deployment in hazardous or uncertain environments because their distributed nature provides inherent robustness to individual robot failures \cite{hazon2008redundancy,yan2013survey,lee2025distributed,wehbe2021probabilistic}. If one robot becomes damaged or disabled, the remaining members of the team can often continue the mission, in contrast to single-robot deployments where such failures may be mission-ending. However, this redundancy raises the question about how risk should be allocated across a team of autonomous agents. In particular, when robots operate in environments with spatially varying hazards, it may be advantageous for some robots to deliberately accept higher levels of risk if doing so improves the overall success of the mission. This observation motivates the study of altruistic decision-making mechanisms, in which individual robots may incur local costs or risks when such maneuvers yield sufficiently large benefits to the collective performance of the team. 

In this paper, we consider a heterogeneous team of risk-aware robots in an exploration mission, where the objective for each team member is to select an exploration trajectory that simultaneously maximizes information gained while minimizing expected agent loss and redundant exploration with other agents.
Inspired by ecological principles of inclusive fitness~\cite{hamilton1963evolution, endler1986natural}, we propose a framework that allows for increased risk-taking by individual robots as long as it benefits the team more than it costs the individual.

Exploration of hazardous environments that would pose undue risks to humans is a well-suited task for multi-robot systems, particularly in applications such as search-and-rescue missions \cite{shree2021exploiting, queralta2020collaborative,drew2021multi,baxter2007multi}. In these settings, robots must reason not only about where to explore, but also about the risks associated with operating in uncertain or dangerous regions. Disregarding risk in such scenarios can lead to higher rates of mission failure and damage to robotic assets, motivating planning strategies that explicitly account for hazard exposure \cite{xiao2020robot,di2022risk}. Recent research has therefore explored decision-making frameworks that optimize criteria beyond minimizing expected risk, enabling systems to better account for rare but high-impact events \cite{jiang2022risk,zhou2022risk}. At the same time, exploration planning should seek to maximize the information gathered from the environment, since efficient information acquisition reduces the travel required to fully cover a map \cite{jadidi2015mutual,carrillo2015autonomous} and guarantees eventual coverage of all reachable unexplored space \cite{julian2014mutual,guo2023patroller,duan2026maximizing}. Balancing information gain with risk exposure is therefore a central challenge in multi-robot exploration.

When multiple robots operate in the same environment, their decisions become inherently coupled. In particular, robots may gather redundant information when visiting similar regions, and risk may be unevenly distributed across the team depending on the trajectories selected by each agent. A natural modeling framework for such interactions is game theory, where robots are treated as decision-making agents seeking to maximize their utilities \cite{lavalle1993game,huo2023task,emery2005game}. Noncooperative game formulations are especially attractive because they permit decentralized implementations in which each robot independently selects its actions. However, when robots optimize only their individual utilities, the resulting behaviors may lead to inefficient collective outcomes that degrade the overall success of the mission, suggesting that some degree of coordination or cooperation is beneficial \cite{roldan2018CompeteOrCooperate,xiao2024task,karam2025resource,butlerCDC2025hamiltonsrule}.

While noncooperative games rely on Nash equilibria as their primary solution concept, cooperative games often lack consensus on solution concepts and can be significantly more computationally demanding \cite{deng1994ComplexityCoop}. One common approach is to modify individual utilities to incorporate terms that capture the effects of an agent’s decisions on others \cite{stirling2005social1,stirling2005social2}. Such formulations combine a private utility component with additional terms that encourage cooperative behavior and improve system-level outcomes \cite{bester1998altruism,de2011altruistic,tobias2023rational}. Importantly, these modified utilities still allow agents to make decisions independently, while implicitly encouraging forms of altruistic behavior through the structure of their incentives. Questions regarding when individuals should incur personal costs to benefit others have also been extensively studied in mathematical biology \cite{ale2013evolution,su2024relational}. In particular, Hamilton \cite{hamilton1963evolution} proposed that individuals sufficiently similar should display altruistic behavior, being willing to incur small losses if doing so provides sufficiently large benefits to their group.

\subsection*{Contributions}
This paper introduces a game-theoretic framework for cooperative risk-aware exploration in heterogeneous multi-robot teams based on principles of ecologically inspired algorithmic altruism. Inspired by Hamilton’s rule from evolutionary biology, we propose an altruistic utility-shaping mechanism that allows robots to internalize the impact of their trajectory choices on the utilities of other agents. By defining inter-agent relatedness through value-based risk sensitivities, we show that the resulting interaction induces a weighted potential game whose maximizers correspond to Pareto-optimal solutions of the underlying exploration problem. We further show that decentralized fictitious play over trajectory choices converges to these equilibria under mild assumptions. Finally, we demonstrate how this framework can be instantiated in a continuous exploration setting that accounts for information gain, sensing redundancy, and spatially varying risk, and verify its real-time efficacy on hardware through robotic experiments using small wheeled robots.

\section{Problem Statement} \label{sec:problem_statement}

Consider a team of $n$ mobile robots indexed by $i~\in~[n]$, where $[n] = \{1,\dots,n\}$, operating over a bounded exploration domain $\Omega \subset \mathbb{R}^d, \quad d \in \{2,3\}$. Consider a discrete-time trajectory through the exploration domain $\Omega$ by a given agent~$i$, where each robot selects a finite-horizon discrete-time trajectory of length $T \in \mathbb{N}$. The feasible trajectory space for agent~$i$ is then given by sets of time-ordered points in the exploration domain $\Omega$ that are constrained by the agents dynamics, given formally as
\begin{equation}
    \begin{aligned}
        \mathcal{X}_i = \{x_i \in [T] \times \Omega  &:  \exists a_i \in \mathcal{A}_i, \forall t \in [T] \\ 
         & \;\; \text{ s.t. } x_i(t+1) = f_i(x_i(t), a_i) \},
    \end{aligned}
\end{equation}
where we assume that agent dynamics can be modeled in discrete-time as
\begin{equation}\label{eq:agent_dynamics}
    x_i(t+1) = f_i(x_i(t), a_i),
\end{equation}
with $a_i \in \mathcal{A}_i$ corresponding to the set of feasible exploration actions, or inputs, that an agent may choose to explore $\Omega$ between time steps $t$ and $t+1$.

The joint finite horizon trajectory profile for all agents is denoted by 
\begin{equation*}
    x = \big(x_1, \dots, x_n\big) \in \mathcal{X} =  \prod_{i \in [n]} \mathcal X_i,
\end{equation*}
where $\mathcal{X} \subseteq  [T] \times \Omega^{n}$.

The environment to be explored by agent~$i$ is characterized by an information utility $I_i : \Omega^n \to \mathbb{R}_{\ge 0}$,
which quantifies the expected information utility gained by agent~$i$ given the trajectories of all agents. Examples of such information densities could include map~\cite{bjorke1996framework}, Gaussian process variance~\cite{o2012gaussian}, or task-relevance scores~\cite{nagami2026vista}.
The total information utility gained by robot $i$ along its trajectory is given by
\begin{equation}
    \mathcal I_i\big(x(\cdot)\big) = \sum_{t=1}^T I_i\!\big(x_i(t), x_{-i}(\cdot)\big).
\end{equation}

Since robots may collect redundant information when their trajectories overlap in space, we introduce a redundancy penalty for joint exploration trajectories. To model this effect, we define a spatial overlap kernel $K : \Omega \times \Omega \to \mathbb{R}_{\ge 0},$
where $K(x_i,x_j)$ measures the degree of sensing overlap between locations $x_i$ and $x_j$.
The redundancy penalty incurred by robot $i$ due to overlap with other robots is defined as
\begin{equation}\label{eq:redundancy}
    \mathcal R_i\big(x(\cdot)\big) = \sum_{j \in [n]} \sum_{t=1}^T \sum_{s=1}^T K\!\big(x_i(t), x_j(s)\big).
\end{equation}

Additionally, each agent is subject to risk, modeled by a risk field $p_i : \Omega^n \to [0,1],$
representing the instantaneous probability of failure or hazard given the joint trajectories of the system.
Each robot $i$ is assigned a positive scalar $\lambda_i >~0$, representing its value or risk sensitivity. The expected cumulative loss incurred by robot $i$ along its trajectory is

\begin{equation}
    \mathcal L_i\big(x_i(\cdot)\big) = \lambda_i \sum_{t=1}^T p_i\!\big(x_i(t), x_{-i}(\cdot)\big).
\end{equation}

Given these metrics, we define the utility accrued by robot $i$ from a joint trajectory profile $x(\cdot\,; \theta)$ as
\begin{equation}
u_i\big(x(\cdot)\big)
=
\mathcal I_i\big(x(\cdot)\big)
-
\alpha\, \mathcal R_i\big(x(\cdot)\big)
-
\mathcal L_i\big(x(\cdot)\big),
\label{eq:individual_utility}
\end{equation}
where $\alpha > 0$ is a weighting parameter penalizing redundant exploration. Note that the utilities are directly coupled across robots through the redundancy term $\mathcal R_i$.

We can describe the joint utility of the system with the vector $u(x(\cdot)) = \big(u_1(x(\cdot)), \dots, u_n(x(\cdot))\big)$ of individual utilities. We consider a social welfare function $W : \mathbb{R}^n \to \mathbb{R}$ that aggregates individual utilities into a system-level objective. A commonly used special case is the utilitarian welfare $W(u) = \sum_{i=1}^n u_i$.
A joint trajectory profile $x^\star(\cdot)$ is said to be \textit{Pareto-optimal} if there exists no feasible $x(\cdot)$ such that $u_i(x(\cdot)) \ge u_i(x(\cdot)),\ \forall i$, with strict inequality for at least one $i$.

Thus, the objective of this work is to design decentralized decision-making mechanisms such that the robots’ trajectory choices converge to Pareto-optimal solutions of the coupled exploration problem, defined concretely as follows:

\begin{problem} \label{prob:optimal}
   Given the exploration domain $\Omega$, information density $I$, agent risk fields $\{p_i\}_{i=1}^n$, and overlap kernel $K$, design a control law over trajectories $x_i(\cdot)$ such that the resulting joint trajectories converge to Pareto-optimal solutions of the coupled exploration problem without centralized coordination. 
\end{problem}
This problem is particularly challenging due to the coupled nature of the utilities induced by overlapping trajectories, heterogeneous risk sensitivities across robots, and the absence of centralized planning or global optimization.
The coupled structure of the utilities in \eqref{eq:individual_utility} suggests that purely self-interested decision making may lead to inefficient or highly redundant exploration behaviors. 
To address this challenge, we draw inspiration from ecological models of altruism, most notably Hamilton’s rule \cite{hamilton1963evolution}, which characterizes when an individual should incur a personal cost to benefit others based on a notion of relatedness. 

In the multi-robot setting, this perspective motivates the introduction of agent-specific altruistic preferences that weight the utilities of neighboring agents according to measures of relative value or shared welfare. We show that embedding such altruistic structure into a game-theoretic learning framework—specifically, fictitious play over trajectory choices—induces dynamics that align individual incentives with system-level efficiency. 

\section{Altruism in Risk-Aware Exploration} \label{sec:altruism_explore}
In this section, we present a game-theoretic framework by which agents can solve Problem~\ref{prob:optimal} by leveraging the coupled utility and asymmetric risk exposure across agents of different value (i.e., $\lambda_i \neq \lambda_j$ for some $i,j \in [n]$). Further, we show that, by formulating the shared objective of the system as an altruistic social utility function, inspired by ecology, we can induce a potential game structure whose equilibrium yeilds a system-level optimal solution according to the social welfare. 

\subsection{Hamilton’s Rule and Altruistic Decision-Making}

In evolutionary biology, altruistic behavior, or kin selection, can be described through Hamilton’s rule \cite{hamilton1963evolution,butlerCDC2025hamiltonsrule,karam2025resource}, which characterizes when an individual should incur a personal cost to benefit another. Specifically, an altruistic action is favored when
$C_i < r_{ij} B_j$,
where $C_i$ denotes the cost incurred by the donor, $B_j$ denotes the benefit received by the recipient, and $r_{ij}$ is a measure of genetic relatedness between the two individuals.

While originally derived to explain genetic selection, Hamilton’s rule can be interpreted more broadly as a decision-making template for agents operating in coupled environments. The rule formalizes the idea that an agent may rationally accept a local loss if it generates sufficiently large benefits for others with whom it shares alignment or interdependence (i.e., relatedness between agents). Importantly, the notion of relatedness need not be genetic, but may instead reflect task relevance, shared objectives, or asymmetric importance within a team.

In multi-robot exploration, individual trajectory choices generate externalities through redundant sensing and interference, while risk is borne locally and asymmetrically. This structure naturally lends itself to an altruistic interpretation, where a robot may deviate from its individually optimal trajectory to reduce redundancy or risk exposure for others, provided the collective benefit outweighs its personal cost as follows. 

Let the cost of agent~$i$ deviating from its current exploration trajectory $x_i$ to an alternative trajectory $x_i'$ be defined as
\begin{equation}
    C_i(x_i, x_i') = u_i(x_i, x_{-i}) - u_i(x_i', x_{-i}), 
\end{equation}
where $x_{-i}$ is the set of trajectories for all agents $j \neq i$. Likewise, define the benefit received by agent~$j$ due to the deviation of agent~$i$ as
\begin{equation}
    B_{ij}(x_i, x_i') = u_j(x_i', x_{-i}) - u_j(x_i, x_{-i}).
\end{equation}
Thus, given a non-negative measure of relatedness between agent $i$ and $j$ as $\gamma_{ij} \geq 0$, we can define a Hamilton's rule-like condition that evaluates the social benefit of a unilateral trajectory deviation by agent~$i$ as
\begin{equation} \label{eq:hamiltions_rule}
     C_i(x_i, x_i') < \sum_{j \neq i}  \gamma_{ij} B_{ij}(x_i, x_i').
\end{equation}
In other words, the trajectory $x_i'$ is strictly preferred over $x_i$ if the collective weighted benefit of agent~$i$ choosing $x_i'$ is greater than the cost incurred by agent~$i$.
When \eqref{eq:hamiltions_rule} does \textit{not} hold for any alternative trajectory $x_i'$, altruistic agents have no incentive to deviate from their trajectories $x_i$ and we say that we have a social Nash Equilibrium, which we define formally as follows:
\begin{definition}[Social Nash Equilibria (SNE)] \label{def:sne}
    If $x^*(\cdot) \in [T] \times \Omega^{n}$ is a SNE, then for all $x(\cdot) \in [T] \times \Omega^{n}$ it holds that
    \begin{equation} \label{eq:SNE}
        v_i(x^*)-v_i(x_i,x_{-i}^*) \geq 0, \; \forall i \in [n],
    \end{equation}
    where
    \begin{equation}
        v_i(x) = u_i(x) + \sum_{j \neq i} \gamma_{ij} \, u_j(x).
        \label{eq:altruistic_utility}
    \end{equation}
\end{definition}
\noindent
Further, if we collect the relatedness weights $\gamma_{ij}$ into a matrix $\Gamma \in \mathbb{R}^{n \times n}_{\geq 0}$, where $\gamma_{ii} = 1$ for all $i$, we can compute the \textit{social utility vector} for the system as
\begin{equation}
    v(x) = \Gamma u(x).
\end{equation}

Given our definition of an SNE, we can imagine several types of games that might arise given a choice of $\Gamma$. If $\Gamma = I$ we recover a noncooperative game, as the social utility vector equals the vector of private utilities, which renders the SNE to be a standard Nash equilibrium by definition. If $\Gamma = 1_{n \times n}$, a matrix of ones, we obtain a cooperative game that aims to maximize the average utility to agents.
If we instead consider $\Gamma$ to be a block diagonal matrix with matrices of ones along its diagonal, $\Gamma = \text{blkdiag}(1_{m_1 \times m_1},1_{m_2 \times m_2},\dots,1_{m_N \times m_N} )$, we obtain an $N$-coalition game, where the agents group into $N$ fixed coalitions, with each coalition cooperating to maximize their average private utility, but being noncooperative across coalitions. 
However, since our objective is more concerned with a fully cooperative system, in this paper, we focus on relatedness structures that help to induce a potential game where decentralized fictitious play can converge to a system-level optimum value according to a social welfare function.

\subsection{Value-Based Relatedness in Risk-Aware Teams}

We propose to encode inter-agent relatedness through the notion of agent value introduced in Section~\ref{sec:problem_statement}. Recall that each agent $i$ is assigned a scalar $\lambda_i > 0$, which captures its relative value, vulnerability, or risk sensitivity within the mission.
Using these values, we define the relatedness between agents $i$ and $j$ as
\begin{equation}\label{eq:value_based_relatedness}
    \gamma_{ij} = \frac{\lambda_j}{\lambda_i}.
\end{equation}
This definition induces asymmetric altruistic preferences, where agents with lower value place greater weight on the utilities of higher-value agents, while higher-value agents are comparatively less influenced by the outcomes of lower-value agents. Intuitively, this structure reflects the principle that preserving or assisting critical agents is more beneficial to the system than prioritizing less critical ones.
As we show next, this utility shaping and relatedness definitions induce a favorable game-theoretic structure.

\subsection{Induced Potential Game Structure} \label{sec:potential_game_structure}

Consider the value-weighted social welfare function
\begin{equation}
\Phi(x(\cdot)) = \sum_{i \in [n]} \lambda_i \, u_i(x(\cdot)).
\label{eq:potential_function}
\end{equation}
We now show that the altruistically shaped game defined by the utilities $\{v_i\}_{i=1}^n$ admits $\Phi$ as a weighted potential function.

\begin{proposition} \label{prop:potential_game}
The game induced by the altruistic utilities \eqref{eq:altruistic_utility} and value-based relatedness \eqref{eq:value_based_relatedness} is a weighted potential game with potential function \eqref{eq:potential_function}.
\end{proposition}

\begin{proof}
Consider a unilateral deviation by agent $i$ from $x_i$ to $x_i'$, holding $x_{-i}$ fixed. The change in the social utility of agent $i$ is
\[
\Delta v_i
=
\Delta u_i
+
\sum_{j \neq i} \frac{\lambda_j}{\lambda_i} \, \Delta u_j,
\]
where $\Delta v_i =  v_i(x(\cdot))-v_i(x_i'(\cdot),x_{-i}(\cdot))$ captures the unilateral deviation of agent~$i$'s trajectory parameterization, and $x_{-i}(\cdot)$ are the trajectories of all agents $j\neq i$, with a similar structure for $\Delta u_i$.
Multiplying both sides by $\lambda_i$ yields
\[
\lambda_i \Delta v_i
=
\sum_{j\in[n]} \lambda_j \, \Delta u_j
=
\Delta \Phi,
\]
which establishes that any unilateral improvement in $v_i$ corresponds to a weighted improvement in the global potential~$\Phi$.
\end{proof}

\edit{As a consequence, every local maximizer of the weighted social welfare function $\Phi$ is an equilibrium of the altruistically shaped game. The converse does not hold in general, where, without additional joint concavity, an SNE is only guaranteed to be a coordinatewise maximizer of $\Phi$ and need not be a local or global maximizer.}
To narrow our focus for analysis, we restrict our attention to only consider the following family of trajectories for agents that can be parameterized by a finite-dimensional vector as follows
\begin{assumption}\label{assume:parameterized_traj}
    Let $\mathcal{X}_i = \{ x_i(\cdot\,;\theta_i): \theta_i\in\Theta_i \} \subset [T] \times \Omega$, where $\Theta_i \subset \mathbb{R}^{m}$.
\end{assumption}
\noindent
Under this assumption, $x_i(\cdot\,;\theta_i)$ denotes a finite-dimensional parametric representation of a function $x_i:[T] \mapsto \Omega$ (such as waypoints, splines, radial basis functions, neural networks, etc.).

Since $\Phi$ is a monotone aggregation of individual utilities, its maximizers correspond to Pareto-optimal solutions of the original exploration problem under the following assumption.
\begin{assumption}\label{assume:strictly_concave}
    \edit{Let $\Theta := \prod_{i\in[n]}\Theta_i$ be convex, and define $\Phi(\theta):=\Phi(x(\cdot\,;\theta))$. The weighted potential $\Phi:\Theta\to\mathbb{R}$ is continuously differentiable and strictly concave with respect to the full joint parameter $\theta=(\theta_1,\ldots,\theta_n)\in\Theta$.}
\end{assumption}
\edit{
A sufficient utility-level condition is that every $u_i(x(\cdot\,;\theta))$ is continuously differentiable and jointly concave in $\theta$, with their positively weighted sum strictly concave. This restriction will be discussed further in the following section for the simplified exploration model.}
\begin{theorem} \label{thm:SNE}
    \edit{Under Assumptions~\ref{assume:parameterized_traj} and \ref{assume:strictly_concave}, a joint trajectory profile $x(\cdot\,;\theta^\star)$ is a pure SNE if and only if $\theta^\star$ is the unique global maximizer of the weighted potential functional~\eqref{eq:potential_function}. Moreover, $x(\cdot\,;\theta^\star)$ is Pareto-optimal with respect to the agents' original utilities.}
\end{theorem}
\begin{proof}
    \edit{First, let $\theta^\star$ be a global maximizer of $\Phi$, and write $x^\star=x(\cdot\,;\theta^\star)$. For any agent $i\in[n]$ and any $\theta_i\in\Theta_i$, Proposition~\ref{prop:potential_game} gives}
    {\small
    \begin{align*}
    \edit{0 \leq \Phi(\theta^\star)-\Phi(\theta_i,\theta_{-i}^\star)
    =\lambda_i\!\left[v_i(x^\star)-v_i\!\left(x_i(\cdot\,;\theta_i),x_{-i}(\cdot\,;\theta_{-i}^\star)\right)\right].}
    \end{align*}
    }
    \edit{Since $\lambda_i>0$, no agent has a beneficial unilateral deviation, and $x^\star$ is a pure SNE.}

    \edit{Conversely, let $x(\cdot\,;\bar\theta)$ be a pure SNE. By Proposition~\ref{prop:potential_game}, $\bar\theta_i$ maximizes $\Phi(\theta_i,\bar\theta_{-i})$ over $\Theta_i$ for every $i$. Continuous differentiability and convexity of $\Theta_i$ therefore imply the blockwise first-order conditions}
    \begin{equation*}
        \edit{\nabla_{\theta_i}\Phi(\bar\theta)^\top(\theta_i-\bar\theta_i)\leq 0,
        \qquad \forall\theta_i\in\Theta_i,\ \forall i\in[n].}
    \end{equation*}
    \edit{Summing these inequalities over all agents yields}
    \begin{equation*}
        \edit{\nabla_{\theta}\Phi(\bar\theta)^\top(\theta-\bar\theta)\leq 0,
        \qquad \forall\theta\in\Theta.}
    \end{equation*}
    \edit{Joint concavity of $\Phi$ then gives}
    \begin{equation*}
        \edit{\Phi(\theta)\leq \Phi(\bar\theta)
        +\nabla_{\theta}\Phi(\bar\theta)^\top(\theta-\bar\theta)
        \leq \Phi(\bar\theta),
        \qquad \forall\theta\in\Theta.}
    \end{equation*}
    \edit{Thus, $\bar\theta$ is a global maximizer of $\Phi$, and strict concavity makes it unique. Finally, if there were a feasible $\tilde\theta$ such that $u_i(x(\cdot\,;\tilde\theta))\geq u_i(x(\cdot\,;\bar\theta))$ for every $i$, with strict inequality for at least one agent, then positivity of the weights $\lambda_i$ would imply $\Phi(\tilde\theta)>\Phi(\bar\theta)$, contradicting global optimality. Hence, $x(\cdot\,;\bar\theta)$ is Pareto-optimal.}
\end{proof}

\edit{Proposition~\ref{prop:potential_game} establishes the weighted-potential structure, while Theorem~\ref{thm:SNE} shows that, under joint strict concavity, the SNE coincides with the unique global maximizer of the weighted social welfare \eqref{eq:potential_function}.}
We now leverage this structure to characterize both the equilibrium behavior of the system and the convergence properties of decentralized learning dynamics when agents iteratively update their planned trajectories.

We focus on a setting in which agents communicate their planned trajectories at each iteration, as described by Algorithm~\ref{alg:altruistic_fp_trajectory}. 
Let $x(\cdot\,; \theta^k) = (x_1(\cdot \, ; \theta_1^k),\dots,x_n(\cdot \, ; \theta_n^k))$ denote the joint planned trajectories at iteration~$k$ with parameterization $\theta^k$. At each iteration, agent $i$ updates $\theta_i$ according to
\begin{equation}
\theta_i^{k+1}
\in
\argmax_{\theta_i \in \Theta_i} \,
v_i\big(x_i(\cdot\,; \theta_i), x_{-i}(\cdot\,; \theta_{-i}^k)\big),
\label{eq:trajectory_best_response}
\end{equation}
where $v_i$ is the altruistic utility defined in \eqref{eq:altruistic_utility} and $\theta_{-i}^k$ is the current parameterization of all agents $j\neq i$ at iteration $k$. We next show that this learning dynamic converges to a Pareto-optimal solution of the original exploration problem under the following additional assumptions.


\begin{algorithm}[t]
\caption{Altruistic Fictitious Play}
\label{alg:altruistic_fp_trajectory}
\begin{algorithmic}[1]
\REQUIRE Initial feasible trajectories $\theta^0_i \in \Theta, \forall i \in [n]$
\FOR{$k = 0,1,2,\dots$}
    \FOR{each agent $i = 1,\dots,N$ \textbf{(sequentially)}}
        \STATE Collect current planned trajectories
        \[
        \theta_{-i}^k = \{ \theta_j \}_{j \neq i}
        \]

        \STATE Compute altruistic best-response trajectory:
        \[
        \theta_i^{k+1}
        =
        \argmax_{\theta_i \in \Theta_i}
        v_i\!\left(
        x_i(\cdot\, ; \theta_i),
        x_{-i}(\cdot\, ; \theta_{-i}^k)
        \right)
        \]

        \STATE Update and broadcast planned trajectory $\theta_i^{k+1}$
    \ENDFOR
\ENDFOR
\end{algorithmic}
\end{algorithm}

\begin{assumption} \label{assume:ficticious_play}
    Let the following properties hold for each agent~$i \in [n]$:
    \begin{enumerate}
        \item \textbf{Feasible Trajectory Sets:}
        \edit{For each agent $i$, the parameter set $\Theta_i$ is nonempty, compact, and convex, and the parameterization $\theta_i\mapsto x_i(\cdot\,;\theta_i)$ is continuous. Consequently, the feasible trajectory space $\mathcal X_i$ is compact.}
    
        \item \textbf{Continuity:}
        Each individual utility function $u_i(x)$ is continuous in the joint trajectory profile $x \in \mathcal X$.
    
    
        \item \textbf{Exact Best Responses:}
        At each iteration, agents compute an exact best response as defined in \eqref{eq:trajectory_best_response}.
    
        \item \textbf{Sequential Updates:}
        All agents update their trajectories sequentially at each iteration.
    \end{enumerate}
\end{assumption}

\begin{theorem}
\label{thm:convergence}
\edit{Let Assumptions~\ref{assume:parameterized_traj}-\ref{assume:ficticious_play} hold. Then every limit point of the sequence $\{x(\cdot \, ; \theta^k)\}$ generated by the trajectory-space best-response update \eqref{eq:trajectory_best_response} through Algorithm~\ref{alg:altruistic_fp_trajectory} is an SNE and corresponds to the unique global maximizer of the weighted social welfare \eqref{eq:potential_function}. Consequently, every limit point is Pareto-optimal with respect to the original utility vector.}
\end{theorem}

\begin{proof}
\edit{By Proposition~\ref{prop:potential_game}, maximizing $v_i$ with respect to $\theta_i$ while holding $\theta_{-i}$ fixed is equivalent to maximizing $\Phi$ over the same block. Assumption~\ref{assume:strictly_concave} implies strict concavity in every block, so Assumptions~\ref{assume:ficticious_play}.1 and \ref{assume:ficticious_play}.2 ensure that each best response exists and is unique.}

\edit{Let $\theta^{k,0}=\theta^k$, let $\theta^{k,i}$ denote the joint parameter immediately after agent $i$ updates during sweep $k$, and let $\theta^{k,n}=\theta^{k+1}$. By Assumption~\ref{assume:ficticious_play}.3, we have each exact best response satisfies}
\begin{equation*}
\edit{\Phi(\theta^{k,i})-\Phi(\theta^{k,i-1})
=\lambda_i\!\left[v_i(\theta^{k,i})-v_i(\theta^{k,i-1})\right]\geq 0.}
\end{equation*}
\edit{Thus, under the sequential updates in Assumption~\ref{assume:ficticious_play}.4, $\{\Phi(\theta^k)\}$ is nondecreasing. It is bounded above because $\Theta=\prod_i\Theta_i$ is compact and $\Phi$ is continuous, and hence it converges to some finite $\Phi^\star$. Since the increase over a complete sweep is the sum of finitely many nonnegative blockwise increases, every blockwise increase converges to zero.}

\edit{Let $\bar\theta$ be an accumulation point of $\{\theta^k\}$ and choose a subsequence $\theta^{k_\ell}\to\bar\theta$. The unique best-response map for each block is continuous by the maximum theorem. If agent~$1$ had a profitable deviation at $\bar\theta$, continuity would give a uniformly positive increase in $\Phi$ at the first update along this subsequence, contradicting the vanishing blockwise increases. Hence, $\bar\theta_1$ is agent~$1$'s unique best response, and $\theta^{k_\ell,1}\to\bar\theta$. Applying the same argument successively to agents $2,\ldots,n$ shows that every $\bar\theta_i$ is a best response to $\bar\theta_{-i}$. Therefore, $x(\cdot\,;\bar\theta)$ is an SNE. Theorem~\ref{thm:SNE} then implies that $\bar\theta$ is the unique global maximizer of $\Phi$ and that the associated trajectory profile is Pareto-optimal.}
\end{proof}

\edit{Joint strict concavity in Assumption~\ref{assume:strictly_concave} is what promotes an SNE from a coordinatewise optimum to the unique global maximizer of \eqref{eq:potential_function}. If this joint condition is replaced by strict concavity only in each agent's own parameter block, Algorithm~\ref{alg:altruistic_fp_trajectory} still admits the following weaker characterization; however, no local, global, or Pareto-optimality conclusion follows in general.}

\begin{corollary} \label{cor:local_SNE}
\edit{Let Assumptions~\ref{assume:parameterized_traj} and \ref{assume:ficticious_play} hold, and suppose that $\Phi$ is continuously differentiable and, for every $i$ and fixed $\theta_{-i}$, strictly concave in $\theta_i$. Then every accumulation point of the sequence generated by Algorithm~\ref{alg:altruistic_fp_trajectory} is an SNE and a coordinatewise maximizer of $\Phi$. 
}
\end{corollary}
\begin{proof}
\edit{Blockwise strict concavity guarantees a unique exact best response for every agent. The monotone-potential and accumulation-point argument in the proof of Theorem~\ref{thm:convergence}, up to the conclusion that every agent is playing a best response, applies without joint concavity. Hence, every accumulation point is an SNE. Proposition~\ref{prop:potential_game} then implies that it is a coordinatewise maximizer of $\Phi$. 
}
\end{proof}

\edit{The above results illustrate how altruism is not solely a behavioral preference, but also a mechanism for aligning decentralized decision making with system-level objectives. By weighting the utilities of teammates according to agent value, each robot accounts for how its trajectory affects the rest of the team, re-framing the exploration problem into a potential game. Under Assumption~\ref{assume:strictly_concave}, its SNE is the unique Pareto-optimal maximizer of the weighted social welfare; without joint concavity, the equilibrium guarantee is limited to coordinatewise optimality.}
With these results, we can now illustrate the application of this risk-aware exploration framework using a simplified setting in which agent trajectories are modeled as sequences of waypoints.

\section{Risk-Aware Map Exploration Model}
While the analysis in Section~\ref{sec:altruism_explore} is general to any agent utilities and trajectory parameterizations that satisfy a combination of Assumptions~\ref{assume:parameterized_traj}-\ref{assume:ficticious_play}, we instantiate in this section the coupled exploration game introduced in Section~\ref{sec:problem_statement} using a utility model that is discretized over the continuous exploration domain. This discretization enables a more computationally tractable setting for individual agents, while also simplifying the interpretation of the exploration results. 

Additionally, since we are interested in examining agent trade-offs between the risk and reward of exploring unknown regions, and how the altruistic utility mixing model might encourage some agents to take greater risks for the benefit of the team, we define information and risk models that explicitly incorporate that risk-reward trade-off as agents traverse unknown regions with potentially unknown hazardous conditions.

\subsection{Coupling Information and Risk in Exploration} \label{sec:coupled_risk_info_model}
A challenging aspect of the exploration problem is that, by the nature of exploration, agents must enter and traverse regions with unknown or uncertain risks. In this setting, there can be a natural trade-off between the amount of information gained and the amount of risk incurred by an agent to obtain that information, where agents may learn the most by exploring completely unknown regions while simultaneously experiencing higher risk due to uncertainties about the unknown environment. 

Thus, to explore this risk-reward trade-off in a setting where some agents may intentionally take on additional risk to reduce uncertainty for other agents, we consider the following simplified risk model for agents
\begin{equation}
    p_i(x(t)) = \mu(x_i(t)) + \beta_i \xi(x_i(t), x_{-i}(\cdot)),
\end{equation}
where $\mu: \Omega \rightarrow \mathbb{R}_{\geq 0}$ models the average (or known) risk and $\xi$ models uncertain (or unknown) risk given trajectories of all agents, with $\beta_i \geq 0$ denoting agent~$i$'s tolerance to unknown risk. We consider a deterministic risk profile defined by known hazards, where the known risk to agent~$i$ of a given trajectory, parameterized by $\theta_i$, is defined by its proximity to hazards as
\begin{equation}
\mu(x_i(t\,; \theta_i))
=
\sum_{r=1}^{R}
\exp\!\left(-\frac{\|x_i(t\,; \theta_i) - h_r\|^2}{2\sigma_r^2}\right),
\label{eq:known_risk_field}
\end{equation}
where $h_r \in \Omega$ denotes hazard centers.

In this paper, we consider a time-varying uncertainty, where, when an agent explores an uncertain region, the uncertainty of that region is reduced as new information is obtained. In this sense, we may consider uncertainty reduction itself as an information utility for agents based on area coverage as follows. Consider the coverage field defined by agent trajectory histories, where the coverage density at any given point $q \in \Omega$, given past and current trajectories of agents up to time $T_0 + T$, is 
\begin{equation}
    \rho\big(q, x(\cdot)\big) = \sum_{i \in [n]} \sum_{t=1}^{T_0+T} K(q, x_i(t)), 
\end{equation}
where $T_0$ is the current time, $T$ is the current trajectory length, and $K : \Omega \times \Omega \to \mathbb{R}_{\ge 0}$ is the overlap kernel
\begin{equation}
    K(x,y) = \exp\left( -\frac{\Vert x - y \Vert^2}{2 \sigma_K^2}\right),
\end{equation}
which, for simplicity, is the same overlap kernel we will use for the redundancy term in \eqref{eq:redundancy}. Thus, to model the uncertainty of a given point, we consider an inverse metric of the coverage density
\begin{equation}
    \xi\big(q , x(\cdot)\big) = \exp\left(-\rho\big(q , x(\cdot)\big)\right),
\end{equation}
where uncovered or unobserved points evaluate to $\xi \approx 1$ and well covered points result in $\xi \rightarrow 0$.
Under this uncertainty model, we say that an agent's information utility for the trajectory $x_i(\cdot\,; \theta_i)$ is the marginal decrease in Shannon entropy $H(q) = -\xi(q) \log\xi(q)$ for the entire uncertainty field. To quantify this marginal decrease, we first define the total entropy as
\begin{equation}
    \Xi\big(x(\cdot)\big) = \int_{\Omega} -\xi\big(q , x(\cdot)\big) \log\xi\big(q , x(\cdot)\big) dq.
\end{equation}
Then, the marginal decrease in map uncertainty given a trajectory $x_i(\cdot \,; \theta_i)$ is
\begin{align*}\label{eq:true_marginal_gain}
   \Delta \Xi_i &=  \Xi\left(\big(x_{-i}, x_i^{T_0}\big)\right) - \Xi\left(\big(x_{-i}, x_i^{T_0}, x_i(\cdot\,; \theta_i)\big)\right), \\
   &= \int_{\Omega} \left[-\xi_{-i}(q) \log \xi_{-i}(q) + \xi(q) \log\xi(q\big) \right]dq
\end{align*}
where $x_{-i}$ is the combined trajectories of all agents $j \neq i$, including current trajectories from $T_0$ to $T_0+T$, $x_{i}^{T_0}$ is the trajectory history of agent~$i$ up to $T_0$, and $\xi_{-i}(q)$ is the uncertainty field without the additional trajectory $x_i(\cdot \,; \theta_i)$. Since $\xi_{-i}(q) = e^{-\rho_{-i}(q)}$, we can rewrite the uncertainty field function as $\xi(q) = \xi_{-i}(q)e^{-\rho_i(q)}$,
where
\begin{equation*}
    \rho_i(q) = \sum_{t = 1}^{T_0+T} K\big(q,x_i(\cdot\,; \theta_i)\big). 
\end{equation*}
Using the Taylor expansion $e^{-\rho_i(q)} \approx 1 - \rho_i(q)$, we have $\xi(q) \approx \xi_{-i}(q)(1 - \rho_i(q))$, thus
\begin{equation*}
    \delta \xi = \xi(q) - \xi_{-i}(q) \approx - \xi_{-i}(q)\rho_i(q).
\end{equation*}
We can then expand entropy using first-order variation
\begin{align*}
    \Delta \Xi &\approx \int_{\Omega} \frac{\partial H(q)}{\partial \xi} \delta\xi(q) dq   \\
    &= \int_{\Omega}(1+ \log\xi_{-i}(q))\xi_{-i}(q) \rho_i(q) dq,
\end{align*}
where this approximation now allows us to linearly separate the contribution of each trajectory point $x_i(t, \theta_i)$.

Thus, we compute the marginal uncertainty reduction of a given trajectory point $x_i(t \,; \theta_i)$ as
\begin{align}
    I_i\big(x_{-i}(\cdot), x_i(t \,; \theta_i)\big)
    &= \int_{\Omega} \kappa_{-i}(q) K\big(q,x_i(t \,; \theta_i)\big) dq,
    \label{eq:info_utility_uncertainty}
\end{align}
where
\begin{equation}
   \kappa_{-i}(q,x_{-i}) = (1+ \log\xi_{-i}(q,x_{-i}))\xi_{-i}(q,x_{-i}).
\end{equation}

Therefore, with this model of information utility, we have a coupled notion of risk and reward through the uncertainty field $\xi$, as agents must manage both information incentives by lowering uncertainty countered with potential loss according to the risk penalty. Note also that since we have included the current trajectories of other agents in the evaluation of future risk, we have also coupled the information and risk utilities for all agents. Importantly, since our information, redundancy, and risk models are smooth in $\theta_i$, we can compute the gradient of the utility $u_i$ (and, by extension, the gradient of the social utility $v_i$), which is required for agent-level trajectory updates in Algorithm~\ref{alg:altruistic_fp_trajectory}, where the full expression of the utility for agent~$i$ is given by
\begin{align}
    u_i\big(x(\cdot)\big)
    &= \sum_{t=T_0+1}^{T_0+T} \int_{\Omega} \kappa_{-i}(q,x_{-i}(\cdot)) K\big(q,x_i(t \,; \theta_i)\big) dq \nonumber\\
    & \quad - \sum_{j \in [n]} \sum_{t=T_0+1}^{T_0+T} \sum_{s=T_0+1}^{T_0+T} K\!\big(x_i(t \,; \theta_i)\big), x_j(s)\big) \nonumber \\
    & \quad - \lambda_i \sum_{t=T_0+1}^{T_0+T} \mu(x_i(t \,; \theta_i)) + \beta_i \xi(x_i(t \,; \theta_i)). \label{eq:full_utility}
\end{align}

\subsection{Numerical Best-Response Computation with Waypoints}

To simplify trajectory parameterization and gradient computation, we directly parameterize each agent’s discrete-time trajectory by a sequence of waypoints. Specifically, for agent $i$ we define
\begin{equation}
\theta_i = [w_{i,1}^\top, \dots, w_{i,T}^\top]^\top,
\quad x_i(t; \theta_i) = w_{i,t}, \; t = 1, \dots, T,
\end{equation}
where $w_{i,t} \in \Omega$ denotes the robot’s planned position at time step $t$.

To enforce basic kinematic feasibility, we constrain successive waypoints to satisfy
\begin{equation}
\|w_{i,t+1} - w_{i,t}\| \leq \Delta w_i^{\max},
\qquad \forall t = 1, \dots, T-1,
\label{eq:motion_constraint}
\end{equation}
where $\Delta w_i^{\max}$ is the maximum distance robot $i$ can travel between timesteps. The feasible parameter set $\Theta_i$ is therefore defined by the workspace constraint $w_{i,t} \in \Omega$ and the inter-waypoint distance constraints in~\eqref{eq:motion_constraint}.

At iteration $k$ of Algorithm~1, agent $i$ computes an approximate best response by solving
\begin{equation*}
\max_{\theta_i \in \Theta_i}
v_i\big(x_i(\cdot\,; \theta_i), x_{-i}(\cdot\,; \theta_{-i}^k)\big),
\end{equation*}
using projected gradient ascent. The gradient of $v_i$ with respect to waypoint $w_{i,t}$ is given by
\begin{equation}\label{eq:grad_simple}
    \nabla_{w_{i,t}} v_i = \nabla_{w_{i,t}} u_i +  \sum_{j \neq i} \gamma_{ij} \, \nabla_{w_{i,t}}u_j,
\end{equation}
where
\begin{align}
\nabla_{w_{i,t}} u_i
&=
\nabla_{w_{i,t}} I_i(w_{i,t}, x_{-i})
-
\lambda_i \nabla_{w_{i,t}} p_i(w_{i,t},x_{-i})
\nonumber \\
&\quad
-
\alpha \sum_{j \neq i} \sum_{s=T_0+1}^{T_0+T}
\nabla_{w_{i,t}} K(w_{i,t}, x_j(s)). \label{eq:grad_ui}
\end{align}
and 
\begin{align}
    \nabla_{w_{i,t}} u_j &= \nabla_{w_{i,t}} I_j - \lambda_j\beta_j\sum_{s=T_0+1}^{T_0+T} \nabla_{w_{i,t}} \xi(x_j(s),x_{-j}) \nonumber \\
    & \quad
    - \alpha \sum_{s=T_0+1}^{T_0+T} \nabla_{w_{i,t}} K\!\big(x_j(s), w_{i,t}\big), \label{eq:grad_uj}
\end{align}
where, it should be noted that $x_{-j}$ includes the waypoint $w_{i,t}$.

For the specific fields defined in Section~\ref{sec:coupled_risk_info_model}, the spatial gradients required in~\eqref{eq:grad_ui} are given by
\begin{equation*}
    \nabla_{w_{i,t}} I_i(w_{i,t},x_{-i})
=
\int_{\Omega}  \kappa_{-i}(q,x_{-i}) \nabla_{w_{i,t}} K\big(q,w_{i,t}\big) dq,
\end{equation*}
\begin{align*}
    \nabla_{w_{i,t}} p_i(w_{i,t},x_{-i})
&= \nabla_{w_{i,t}}\mu(w_{i,t}) + \beta_i \nabla_{w_{i,t}}\xi(w_{i,t},x_{-i}) 
\end{align*}
with
\begin{align*}
   \nabla_{w_{i,t}}\mu(w_{i,t}) &= \sum_{r=1}^{R}
\nabla_{w_{i,t}}K(w_{i,t}, h_r)
\end{align*}
and
\begin{align*}
    \nabla_{w_{i,t}}\xi(w_{i,t},x_{-i}) & = \frac{\partial\xi(w_{i,t},x_{-i})}{\partial \rho(w_{i,t},x_{-i})} \nabla_{w_{i,t}}\rho(w_{i,t},x_{-i}) \\ 
    &=  -\xi(w_{i,t},x_{-i})\nabla_{w_{i,t}}\rho(w_{i,t},x_{-i}) 
\end{align*}
where
\begin{equation*}
    \nabla_{w_{i,t}}\rho(w_{i,t},x_{-i}) =  \sum_{j \neq i} \sum_{s=1}^{T_0+T} \nabla_{w_{i,t}} K(w_{i,t}, x_j(s)),
\end{equation*}
and
\begin{equation*}
    \nabla_x K(x,y) = \frac{1}{\sigma_K^2}K(y,x)(y-x).
\end{equation*}
The additional cross-gradient terms in \eqref{eq:grad_uj} are computed as
\begin{align}
    \nabla_{w_{i,t}} I_j &=
    \sum_{s = 1}^{T_0+T}\int_{\Omega} K\big(q,x_j(s)\big) \frac{\partial \kappa_{-j}(q)}{ \partial \rho_{-j}(q)} \nabla_{w_{i,t}}K\big(q, w_{i,t}\big) dq,
\end{align}
where
\begin{equation*}
    \frac{\partial \kappa_{-j}(q)}{ \partial \rho_{-j}(q)} = \frac{\partial \kappa_{-j}}{ \partial \xi_{-j}}\frac{\partial \xi_{-j}}{ \partial \rho_{-j}}= -(2+\log \xi_{-j}(q))\xi_{-j}(q),
\end{equation*}
and
\begin{equation*}
    \nabla_{w_{i,t}} \xi(x_j(s),x_{-j}) = -\xi(x_j(s),x_{-j}) \nabla_{w_{i,t}}K(x_j(s), w_{i,t}).
\end{equation*}

Consequently, the full parameter gradient is given by
\begin{equation*}
    \nabla_{\theta_i} v_i = \big[ \nabla_{w_{i,1}} v_i, \dots, \nabla_{w_{i,T}} v_i \big]^\top,
\end{equation*}
making the gradient accent parameter update
\begin{equation}
\theta_i^{(\tau+1)}
=
\Pi_{\Theta_i}
\left(
\theta_i^{(\tau)}
+
\eta \nabla_{\theta_i} v_i
\right),
\end{equation}
where $\eta>0$ is the accent rate parameter and $\Pi_{\Theta_i}$ denotes projection onto the feasible set defined by $\Omega$ and~\eqref{eq:motion_constraint}. In practice, this projection is implemented by clamping each waypoint to $\Omega$ and projecting successive waypoint differences onto the ball of radius $\Delta w_i^{\max}$.

The exploration model developed in this section provides a concrete instantiation of the altruistic game-theoretic framework introduced in Section~\ref{sec:altruism_explore}. We next evaluate the resulting planner in simulation and hardware experiments to examine how altruistic utility shaping affects information gathering, risk allocation, and collective exploration performance.

\section{Experiments}

We evaluate the proposed altruistic exploration framework in numerical simulation and on the Robotarium platform. The purpose of the simulation study is to isolate the effect of the altruistic utility shaping on coverage, risk allocation, and inter-agent coordination, while the Robotarium experiment verified that the resulting waypoint planner can be executed by physical unicycle robots using standard low-level controllers and safety filters.

In simulation, robots explore the bounded planar domain $\Omega=[-1.5,1.5]\times[-1,1]$, discretized on a $40\times40$ grid. The known hazard field consists of two Gaussian hazards centered at $(0.5,0.5)$ and $(-0.5,-0.3)$ with standard deviations $\sigma_r$ of~$0.3$ and $0.5$, respectively. Coverage is accumulated using a Gaussian sensing kernel with $\sigma_K=0.1$, and the uncertainty at each grid point is computed as $\xi(q)=\exp(-\rho(q))$, where $\rho(q)$ is the accumulated coverage density, where, for all experiments, the four robots are initialized at $(-1,0)$, $(1,0)$, $(0,0.5)$, and $(0,-0.5)$.

Each robot plans a finite-horizon waypoint trajectory with planning horizon $T=4$, with $dt = 1$ and $\Delta w_i^{\max} = 0.2$. The first waypoint is fixed at the robot's current position, and subsequent waypoints are projected to remain inside $\Omega$ and satisfy the maximum inter-waypoint displacement constraint. At each replanning round, projected gradient ascent is used to approximately maximize the altruistic utility $v_i$ according to \eqref{eq:altruistic_utility} and \eqref{eq:value_based_relatedness}, given the agent value weights $\lambda_{i}$ for $i \in [n]$. We compare a selfish baseline, corresponding to $\Gamma=I$, against the altruistic game with relatedness weights $\gamma_{ij}=\lambda_j/\lambda_i$. Under this choice, lower-valued robots place greater relative weight on protecting higher-valued teammates, while higher-valued robots are less inclined to sacrifice themselves for lower-valued agents.

Figure~\ref{fig:sim_traj} shows representative final trajectories for the selfish and altruistic planners, where the self planner sets $\Gamma = I$ and the altruistic planner sets $\Gamma$ according to \eqref{eq:value_based_relatedness}, with $\lambda = [40, 40, 15, 15]$. Both methods reduce uncertainty across the environment, but the spatial organization of the trajectories differs.
In the selfish case, each robot optimizes only its own information-risk tradeoff, leading to more overlapping paths and less explicit differentiation by agent value. In contrast, the altruistic planner produces trajectories that are more separated and more strongly structured by the value parameters. Lower-valued robots take on exploration near higher-risk regions when doing so improves team coverage and reduces the need for higher-valued robots to enter those regions. Thus, altruism does not simply make the team more conservative; rather, it changes which robots absorb risk.

We note that the trajectories in Fig.~\ref{fig:sim_traj} are waypoint plans, so they may appear piecewise linear rather than dynamically smooth. This representation keeps the best-response optimization computationally efficient, while a lower-level controller can convert the waypoint sequence into executable robot motion.

\begin{figure}[htbp]
    \centering
    \begin{subfigure}[b]{0.9\columnwidth}  
        \centering
        \includegraphics[width=\textwidth]{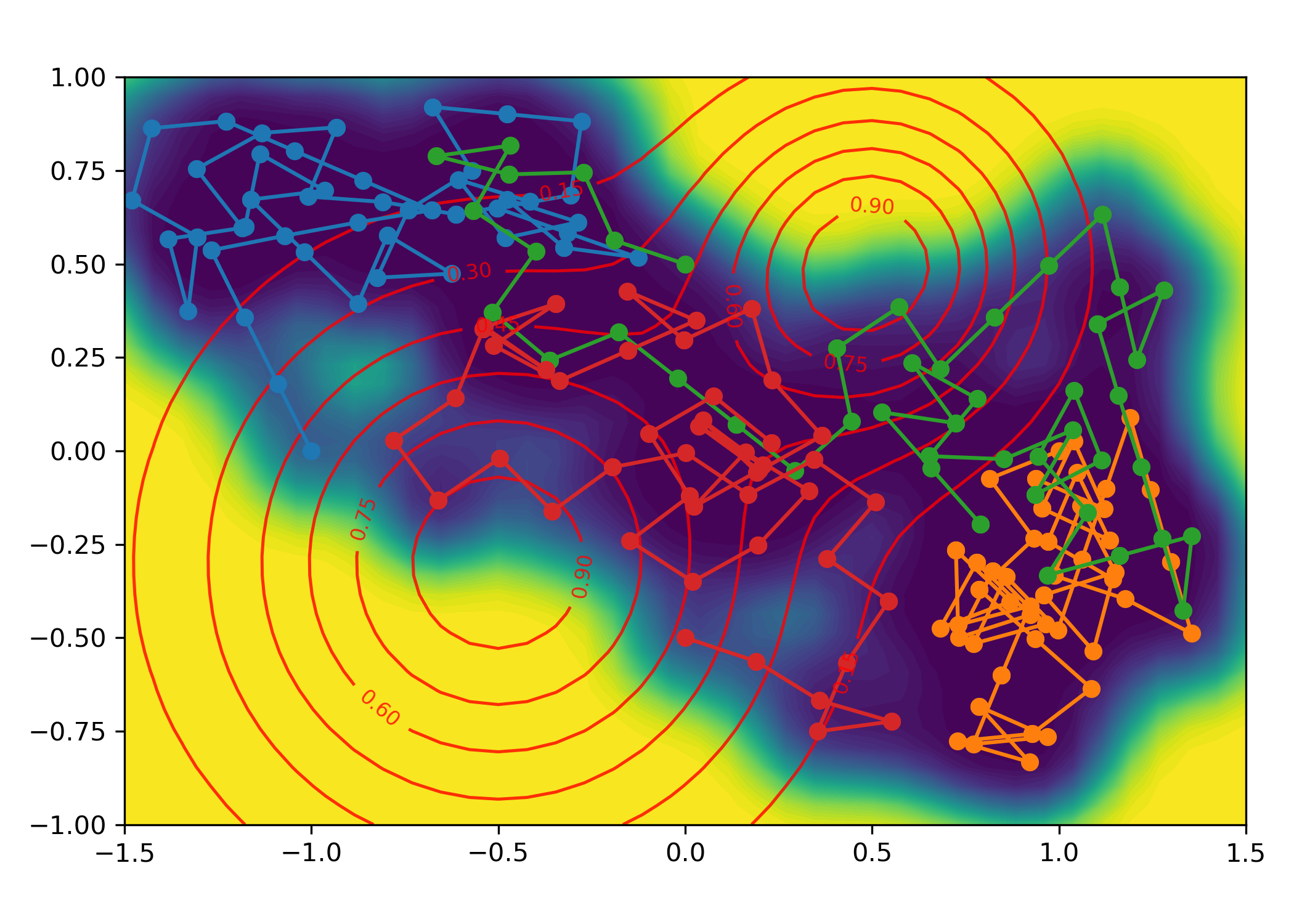}
        \caption{Non-altruistic}
        \label{fig:top}
    \end{subfigure}
    \begin{subfigure}[b]{0.9\columnwidth}
        \centering
        \includegraphics[width=\textwidth]{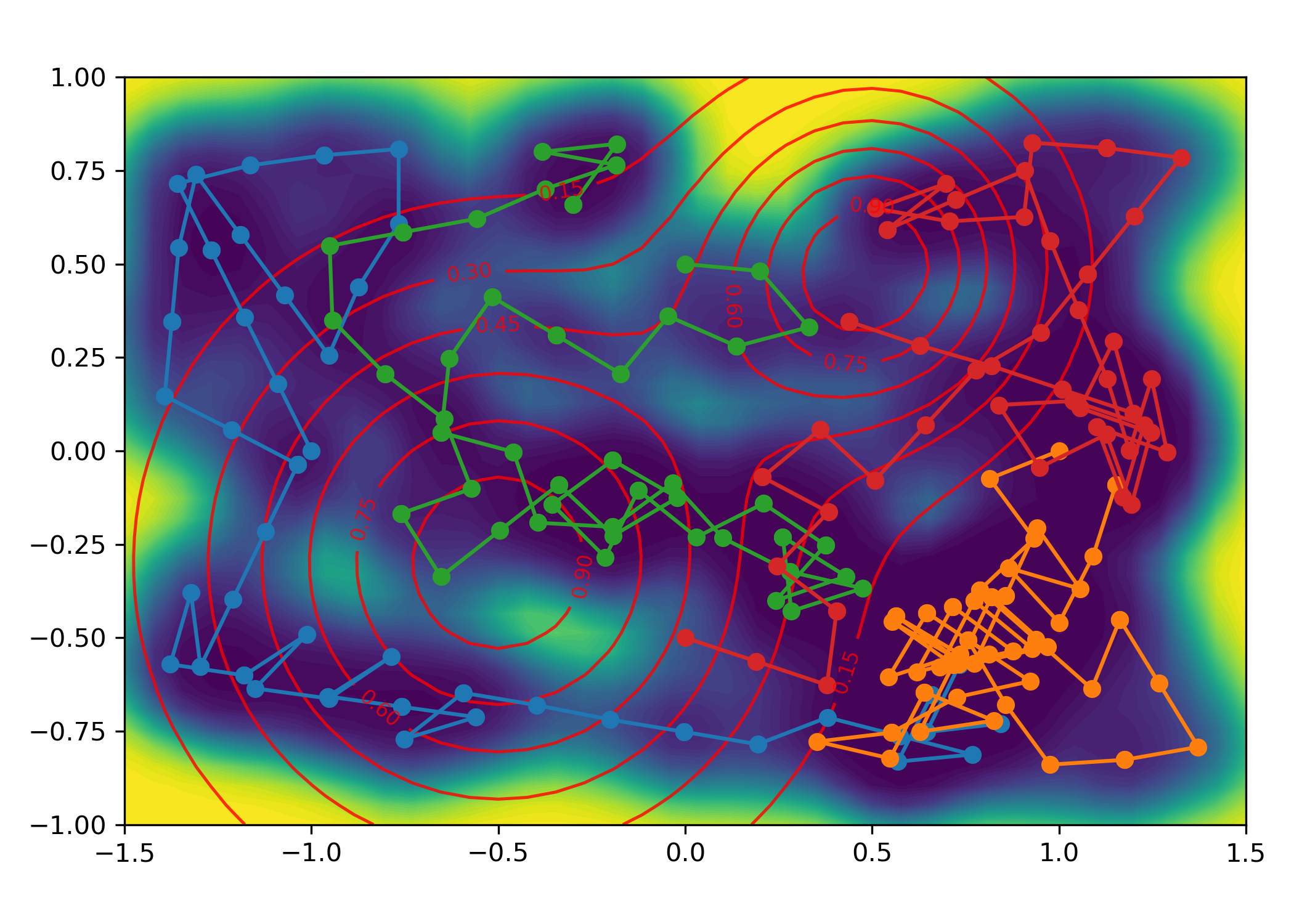}
        \caption{Altruistic}
        \label{fig:bottom}
    \end{subfigure}

    \caption{Representative simulated exploration trajectories for (a) the selfish baseline (i.e., $\Gamma = I$) and (b) the altruistic utility with relationships defined by \eqref{eq:value_based_relatedness}, where $\lambda_{\text{blue}} = \lambda_{\text{orange}} = 40$ and $\lambda_{\text{red}}=\lambda_{\text{green}} = 15$. The filled background shows final map uncertainty, with darker regions corresponding to lower remaining uncertainty. Red contours denote the known hazard field. While both methods explore the domain, the altruistic planner produces more spatially separated trajectories and allocates hazardous exploration according to agent value, allowing lower-valued agents to absorb risk that would otherwise be borne by higher-valued teammates.}
\label{fig:sim_traj}
\end{figure}

To quantify these effects across different importance configurations, we run 50 trials for five configurations of varying risk tolerances and relative agent importance compositions, where randomness enters through the warm-start perturbations used when initializing trajectories before gradient optimization. Figure~\ref{fig:metrics} reports four metrics for selected homogeneous and heterogeneous value profiles: mean uncertainty, raw accumulated risk, value-weighted risk, and minimum pairwise distance. Bars show the sample mean across trials and error bars indicate standard error. The altruistic planner achieves comparable uncertainty reduction to the selfish baseline, indicating that the cooperative risk allocation does not come at the expense of exploration performance. The most consistent improvement appears in minimum pairwise distance, where altruism increases separation between robots and thereby reduces redundant or conflicting exploration. In the higher-value and mixed-value cases, altruism also reduces value-weighted risk, showing that the induced behavior preferentially protects more valuable agents while still allowing the team to gather information.

\begin{figure*}
    \centering
    \includegraphics[width=\linewidth]{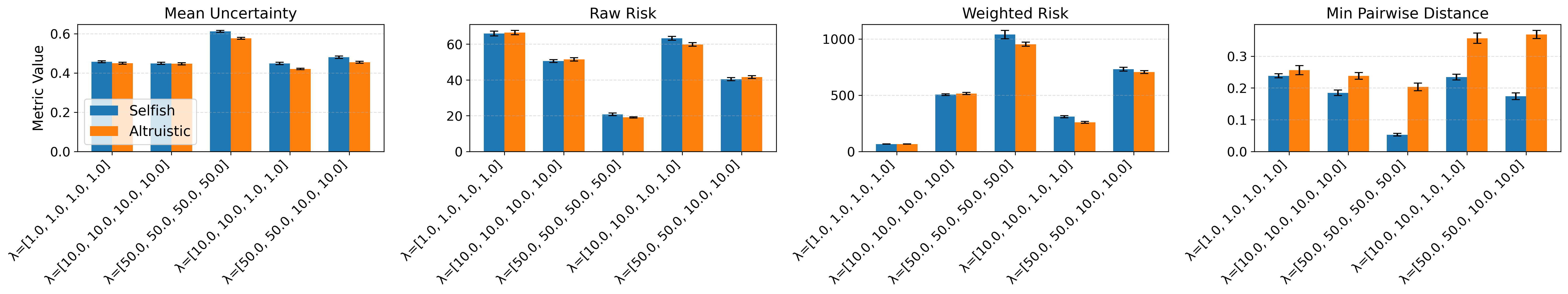}
    \caption{Aggregate simulation performance over 50 trials for selected value configurations. Bars show sample means and error bars show standard error. The altruistic planner maintains comparable uncertainty reduction to the selfish baseline, while generally increasing minimum pairwise distance and reducing redundant exploration. In heterogeneous and high-value configurations, altruism also lowers value-weighted risk, indicating that the relatedness-weighted utility reallocates risk away from more valuable agents.}
\label{fig:metrics}
\end{figure*}

Finally, we implement the planner on hardware in the UCI Robotarium, where the high-level planner generates the same finite-horizon waypoint plans in real-time, while each robot tracks the current waypoint using a single-integrator position controller. Control inputs are then passed through a boundary-aware single-integrator barrier certificate and then mapped to unicycle commands. Figure~\ref{fig:robotarium} shows snapshots from a representative run, where a video of this run may be viewed at \href{https://youtu.be/KigS25qrS4I}{https://youtu.be/KigS25qrS4I}. The robots progressively reduce uncertainty while respecting collision and boundary constraints, demonstrating that the proposed planner can be coupled to standard control primitives without requiring a specialized low-level controller.

\begin{figure*}[t]
    \centering
    
    \begin{subfigure}[b]{0.31\textwidth}
        \centering
        \includegraphics[width=\textwidth]{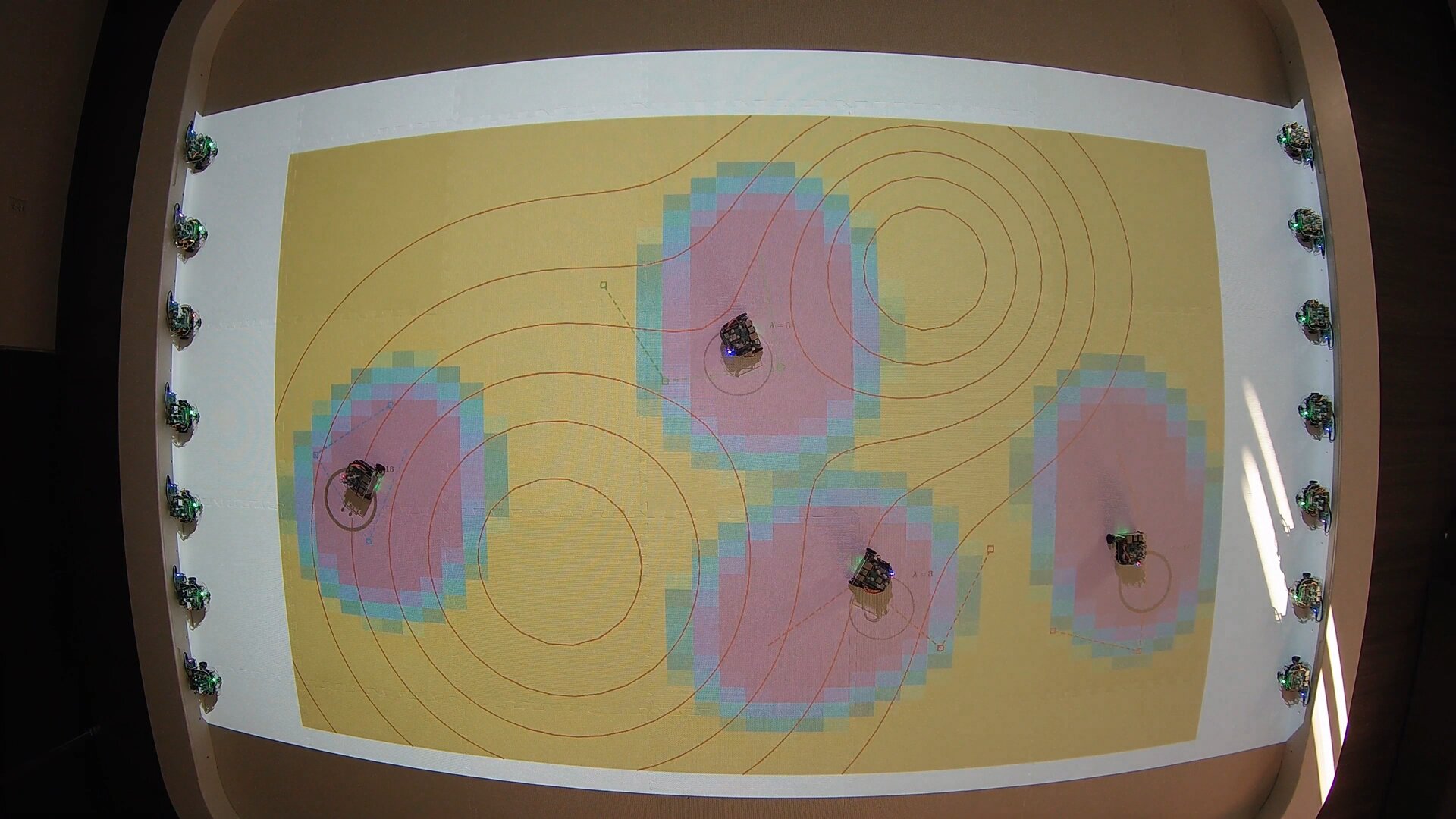}
        \caption*{$t = 5$\,s}
        \label{fig:t0}
    \end{subfigure}
    \hfill 
    \begin{subfigure}[b]{0.31\textwidth}
        \centering
        \includegraphics[width=\textwidth]{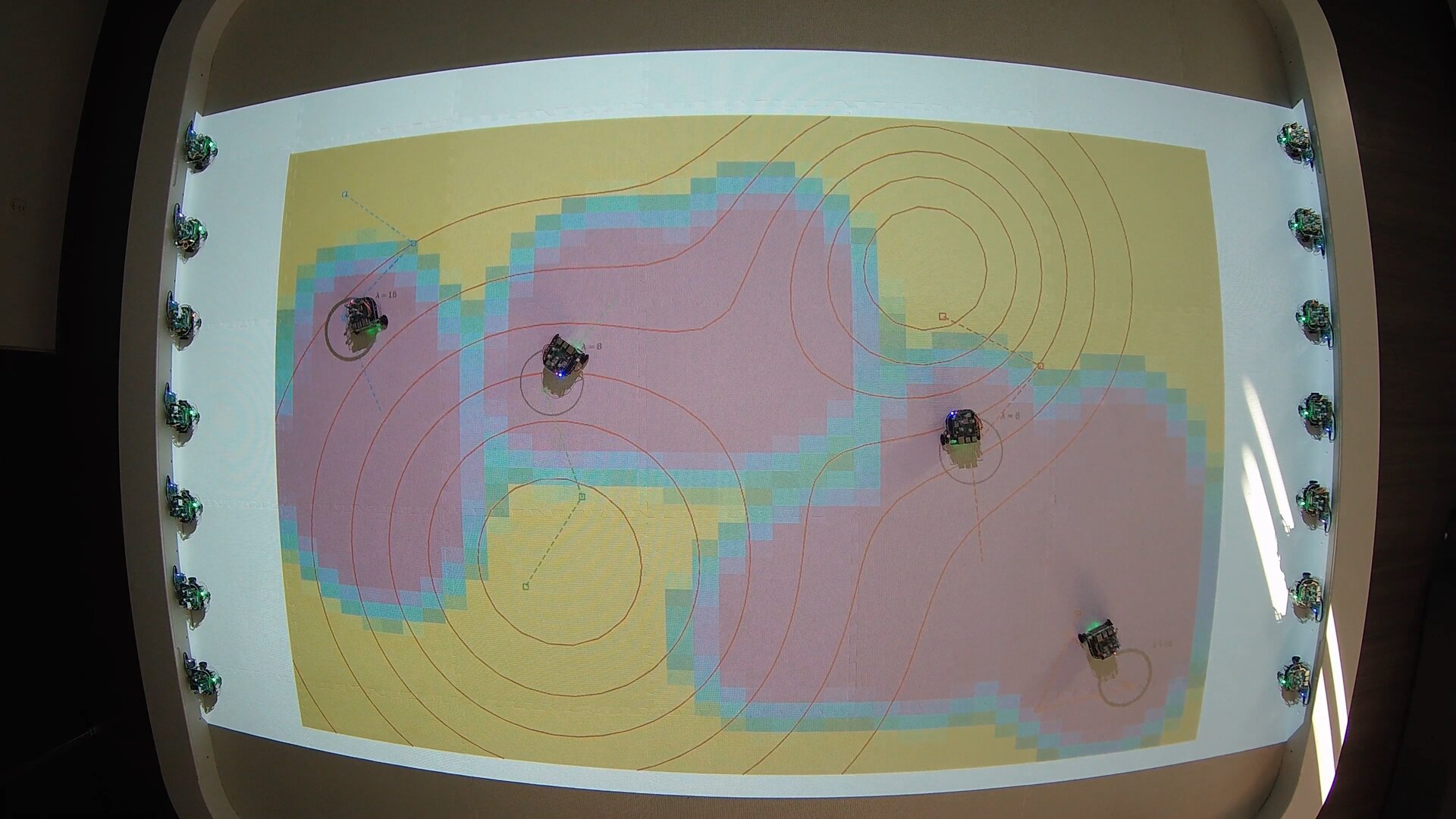}
        \caption*{$t = 10$\,s}
        \label{fig:t1}
    \end{subfigure}
    \hfill
    \begin{subfigure}[b]{0.31\textwidth}
        \centering
        \includegraphics[width=\textwidth]{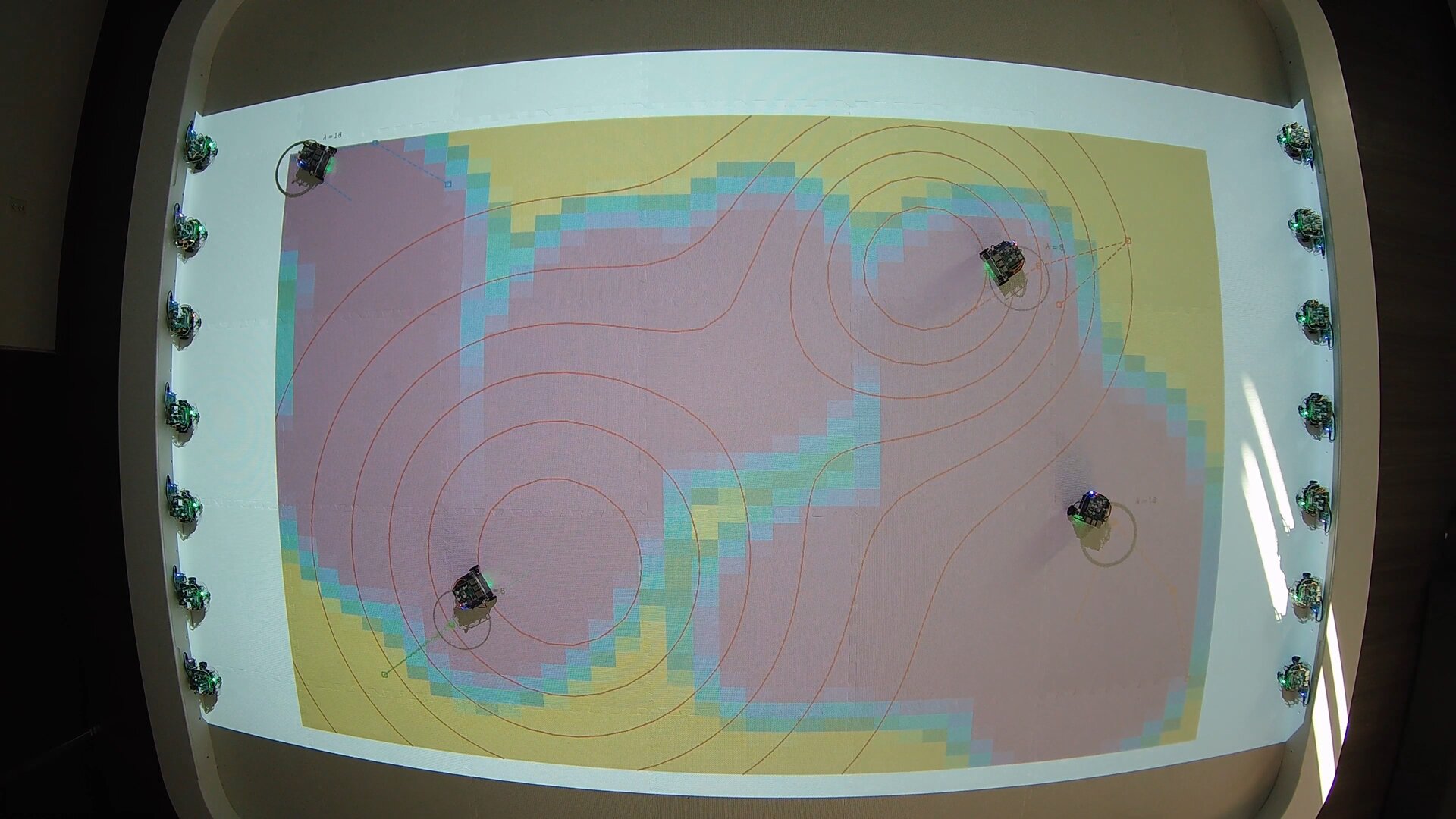}
        \caption*{$t = 15$\,s}
        \label{fig:t2}
    \end{subfigure}
    
    \vspace{0.3cm} 
    \par           
    
    \begin{subfigure}[b]{0.31\textwidth}
        \centering
        \includegraphics[width=\textwidth]{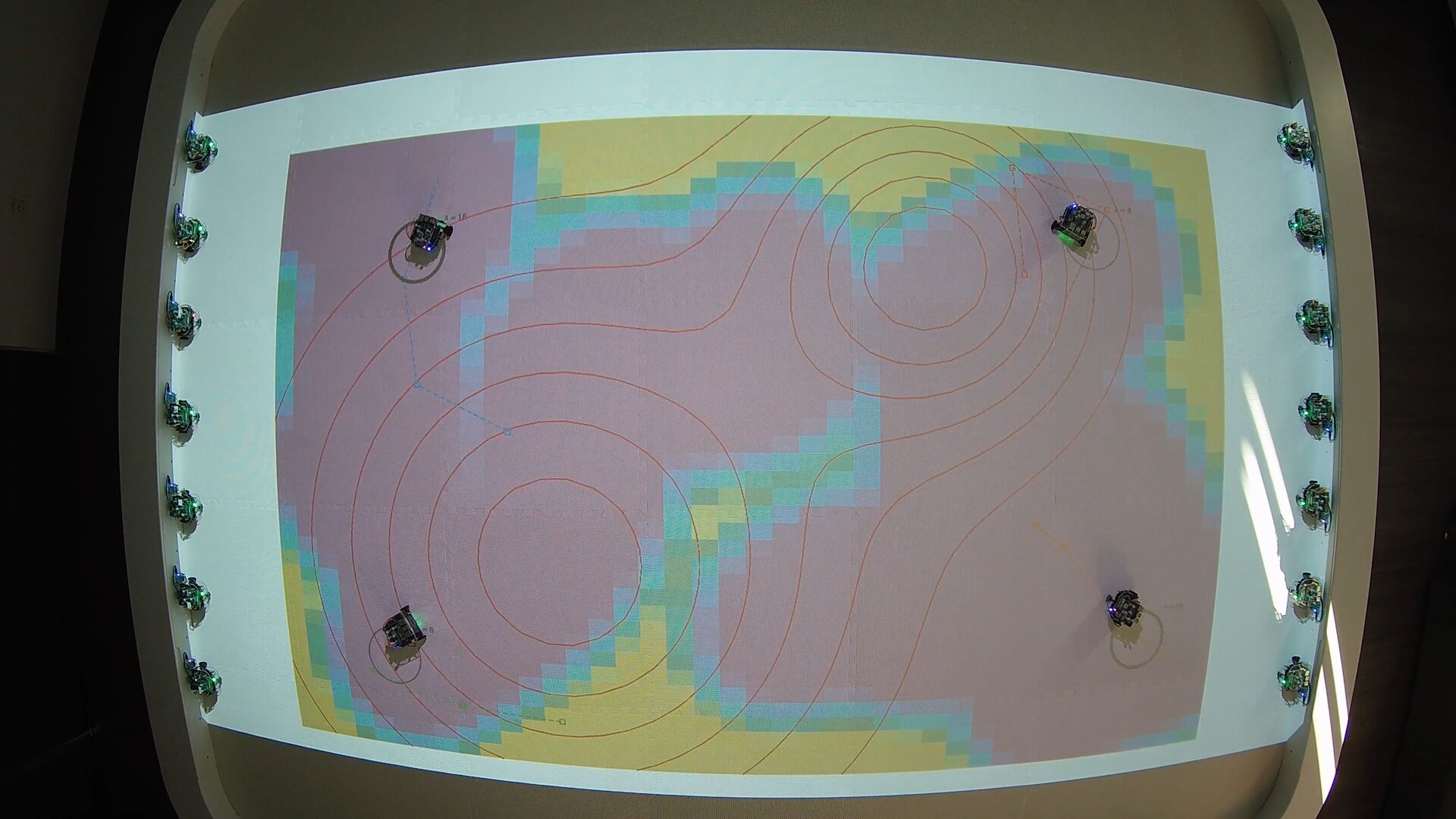}
        \caption*{$t = 20$\,s}
        \label{fig:t3}
    \end{subfigure}
    \hfill
    \begin{subfigure}[b]{0.31\textwidth}
        \centering
        \includegraphics[width=\textwidth]{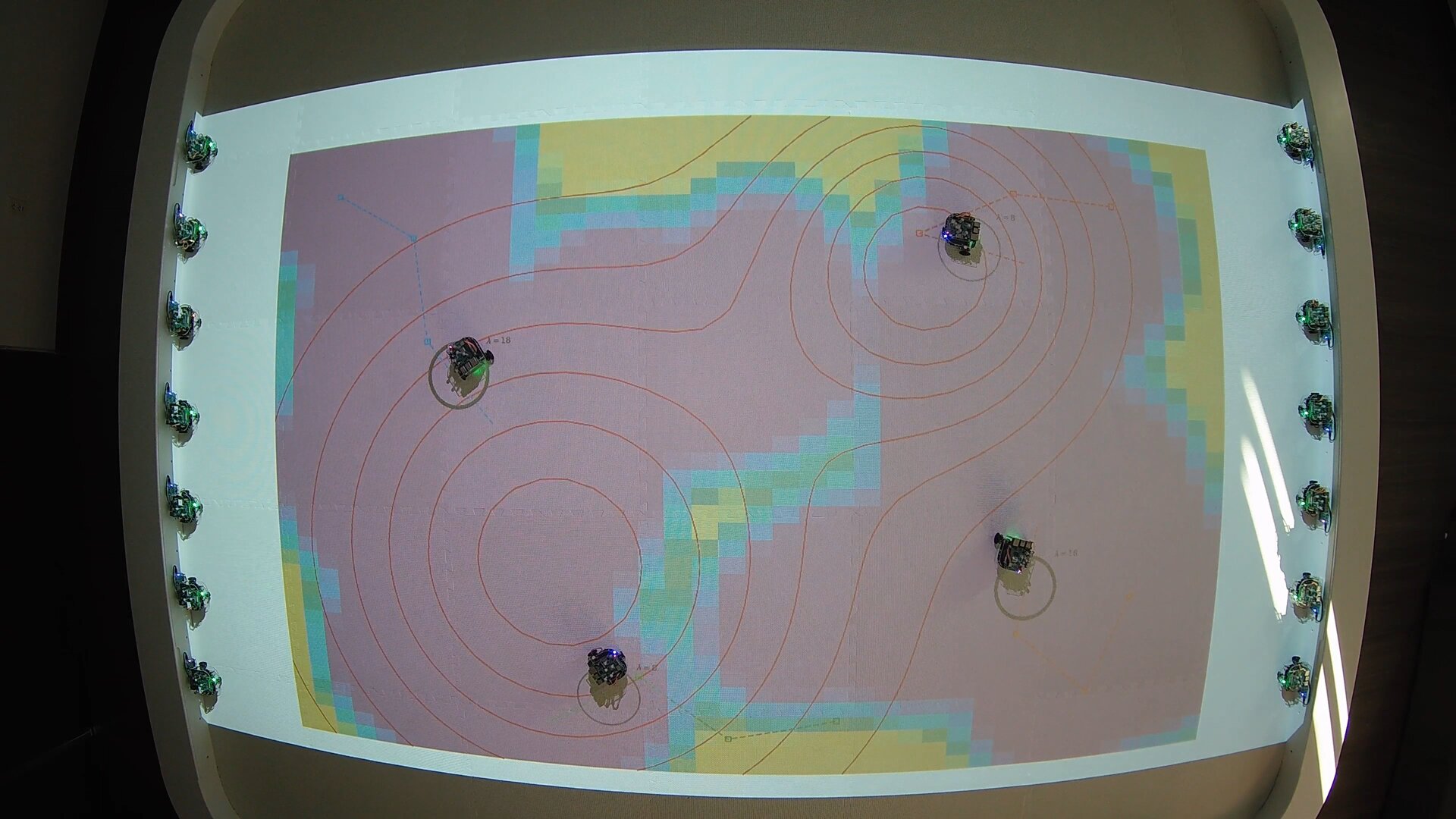}
        \caption*{$t = 25$\,s}
        \label{fig:t4}
    \end{subfigure}
    \hfill
    \begin{subfigure}[b]{0.31\textwidth}
        \centering
        \includegraphics[width=\textwidth]{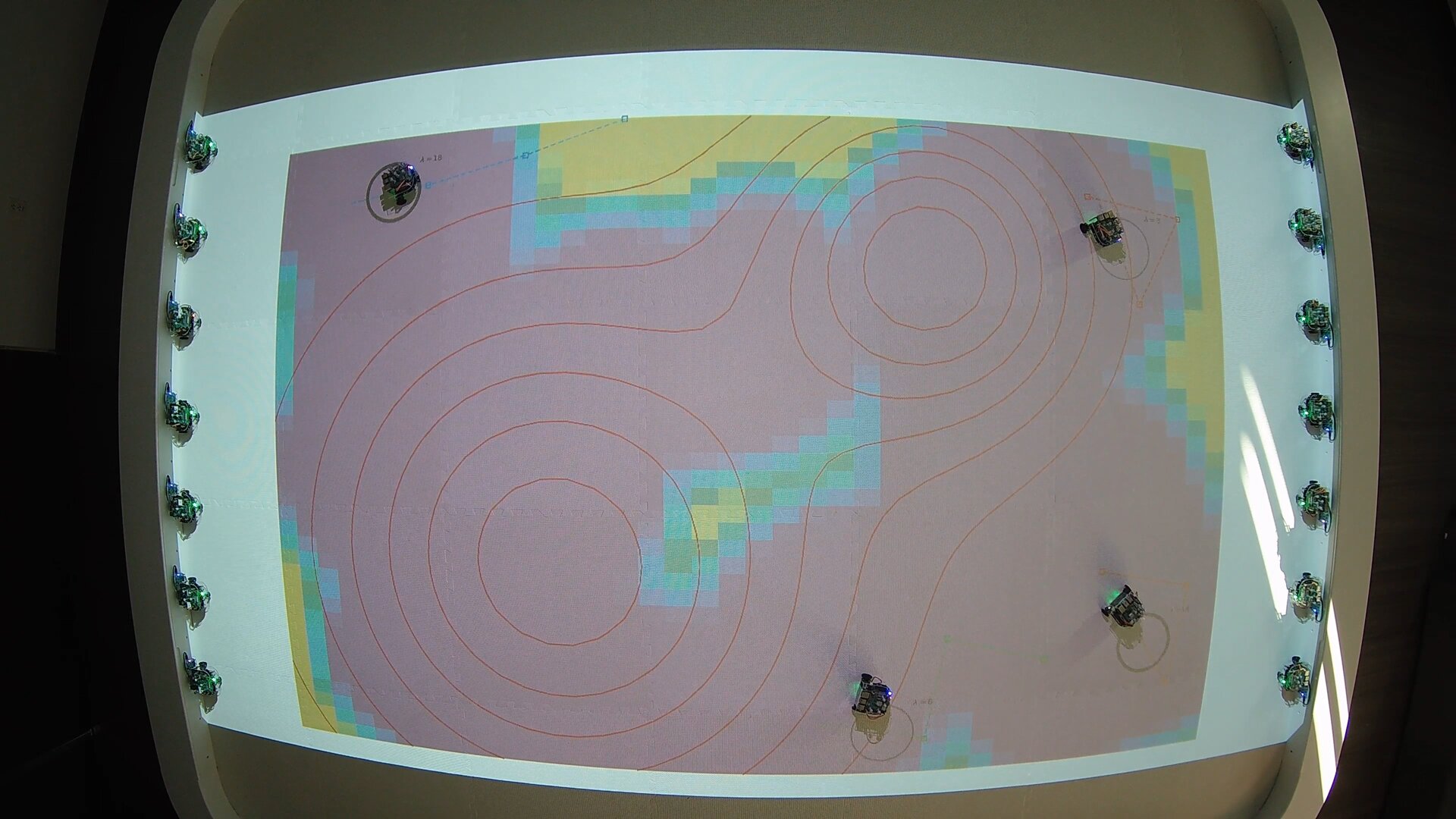}
        \caption*{$t = 30$\,s}
        \label{fig:t5}
    \end{subfigure}

    \caption{Robotarium execution of the altruistic risk-aware exploration planner. The projected visualization shows the online uncertainty estimate together with contour lines for the known hazard field. At each replanning round, the high-level planner produces waypoint trajectories, which are tracked by Robotarium single-integrator controllers with boundary-aware barrier certificates. Over time, the robots reduce uncertainty across the arena while maintaining safe separation and respecting workspace boundaries.}
\label{fig:robotarium}
\end{figure*}

\section{Conclusion}

This paper presents an altruistic game-theoretic framework for risk-aware multi-robot exploration with heterogeneous agents. By shaping each robot's utility with value-dependent relatedness weights, the proposed method allows agents to account for the effect of their decisions on teammates while preserving a decentralized trajectory-optimization structure. In the continuous exploration setting, this framework was implemented with smooth uncertainty, redundancy, and hazard models, enabling projected gradient updates over finite-horizon waypoint trajectories.

The experimental results demonstrate that altruistic utility shaping changes how risk and information-gathering effort are allocated across the team. Compared with the selfish baseline, the altruistic planner produced more spatially separated trajectories, reduced redundant exploration, and in several heterogeneous value configurations lowered value-weighted risk while maintaining comparable uncertainty reduction. Robotarium experiments further showed that the resulting waypoint plans can be executed using standard single-integrator controllers and barrier certificates, with online uncertainty updates from measured robot motion.

Future work will extend the framework to richer robot dynamics, partially known or learned hazard fields, and larger teams with communication constraints. Another important direction is to study adaptive relatedness weights that depend not only on fixed agent value, but also on mission context, remaining energy, robot health, or task-critical capabilities.

\normalem
\bibliographystyle{IEEEtran}
\bibliography{references}

\begin{thebibliography}{10}
\providecommand{\url}[1]{#1}
\csname url@samestyle\endcsname
\providecommand{\newblock}{\relax}
\providecommand{\bibinfo}[2]{#2}
\providecommand{\BIBentrySTDinterwordspacing}{\spaceskip=0pt\relax}
\providecommand{\BIBentryALTinterwordstretchfactor}{4}
\providecommand{\BIBentryALTinterwordspacing}{\spaceskip=\fontdimen2\font plus
\BIBentryALTinterwordstretchfactor\fontdimen3\font minus \fontdimen4\font\relax}
\providecommand{\BIBforeignlanguage}[2]{{%
\expandafter\ifx\csname l@#1\endcsname\relax
\typeout{** WARNING: IEEEtran.bst: No hyphenation pattern has been}%
\typeout{** loaded for the language `#1'. Using the pattern for}%
\typeout{** the default language instead.}%
\else
\language=\csname l@#1\endcsname
\fi
#2}}
\providecommand{\BIBdecl}{\relax}
\BIBdecl

\bibitem{hazon2008redundancy}
N.~Hazon and G.~A. Kaminka, ``On redundancy, efficiency, and robustness in coverage for multiple robots,'' \emph{Robotics and Autonomous Systems}, vol.~56, no.~12, pp. 1102--1114, 2008.

\bibitem{yan2013survey}
Z.~Yan, N.~Jouandeau, and A.~A. Cherif, ``A survey and analysis of multi-robot coordination,'' \emph{International Journal of Advanced Robotic Systems}, vol.~10, no.~12, p. 399, 2013.

\bibitem{lee2025distributed}
H.~Lee and D.~Panagou, ``Distributed resilience-aware control in multi-robot networks,'' in \emph{2025 IEEE 64th Conference on Decision and Control (CDC)}.\hskip 1em plus 0.5em minus 0.4em\relax IEEE, 2025, pp. 3868--3875.

\bibitem{wehbe2021probabilistic}
R.~Wehbe and R.~K. Williams, ``Probabilistic resilience of dynamic multi-robot systems,'' \emph{IEEE Robotics and Automation Letters}, vol.~6, no.~2, pp. 1777--1784, 2021.

\bibitem{hamilton1963evolution}
W.~D. Hamilton, ``The evolution of altruistic behavior,'' \emph{The American Naturalist}, vol.~97, no. 896, pp. 354--356, 1963.

\bibitem{endler1986natural}
J.~A. Endler, \emph{Natural Selection in the Wild}.\hskip 1em plus 0.5em minus 0.4em\relax Princeton University Press, 1986, no.~21.

\bibitem{shree2021exploiting}
V.~Shree, B.~Asfora, R.~Zheng, S.~Hong, J.~Banfi, and M.~Campbell, ``Exploiting natural language for efficient risk-aware multi-robot sar planning,'' \emph{IEEE Robotics and Automation Letters}, vol.~6, no.~2, pp. 3152--3159, 2021.

\bibitem{queralta2020collaborative}
J.~P. Queralta, J.~Taipalmaa, B.~C. Pullinen, V.~K. Sarker, T.~N. Gia, H.~Tenhunen, M.~Gabbouj, J.~Raitoharju, and T.~Westerlund, ``Collaborative multi-robot search and rescue: Planning, coordination, perception, and active vision,'' \emph{IEEE Access}, vol.~8, pp. 191\,617--191\,643, 2020.

\bibitem{drew2021multi}
D.~S. Drew, ``Multi-agent systems for search and rescue applications,'' \emph{Current Robotics Reports}, vol.~2, no.~2, pp. 189--200, 2021.

\bibitem{baxter2007multi}
J.~L. Baxter, E.~Burke, J.~M. Garibaldi, and M.~Norman, ``Multi-robot search and rescue: A potential field based approach,'' in \emph{Autonomous Robots and Agents}.\hskip 1em plus 0.5em minus 0.4em\relax Springer, 2007, pp. 9--16.

\bibitem{xiao2020robot}
X.~Xiao, J.~Dufek, and R.~R. Murphy, ``Robot risk-awareness by formal risk reasoning and planning,'' \emph{IEEE Robotics and Automation Letters}, vol.~5, no.~2, pp. 2856--2863, 2020.

\bibitem{di2022risk}
K.~Di, Y.~Zhou, J.~Jiang, F.~Yan, S.~Yang, and Y.~Jiang, ``Risk-aware collection strategies for multirobot foraging in hazardous environments,'' \emph{ACM Transactions on Autonomous and Adaptive Systems (TAAS)}, vol.~16, no. 3-4, pp. 1--38, 2022.

\bibitem{jiang2022risk}
L.~Jiang and Y.~Wang, ``Risk-aware decision-making in human-multi-robot collaborative search: a regret theory approach,'' \emph{Journal of Intelligent \& Robotic Systems}, vol. 105, no.~2, p.~40, 2022.

\bibitem{zhou2022risk}
L.~Zhou and P.~Tokekar, ``Risk-aware submodular optimization for multirobot coordination,'' \emph{IEEE Transactions on Robotics}, vol.~38, no.~5, pp. 3064--3084, 2022.

\bibitem{jadidi2015mutual}
M.~G. Jadidi, J.~V. Miro, and G.~Dissanayake, ``Mutual information-based exploration on continuous occupancy maps,'' in \emph{2015 IEEE/RSJ International Conference on Intelligent Robots and Systems (IROS)}.\hskip 1em plus 0.5em minus 0.4em\relax IEEE, 2015, pp. 6086--6092.

\bibitem{carrillo2015autonomous}
H.~Carrillo, P.~Dames, V.~Kumar, and J.~A. Castellanos, ``Autonomous robotic exploration using occupancy grid maps and graph slam based on shannon and r{\'e}nyi entropy,'' in \emph{Proceedings of the International Conference on Robotics and Automation (ICRA)}.\hskip 1em plus 0.5em minus 0.4em\relax IEEE, 2015, pp. 487--494.

\bibitem{julian2014mutual}
B.~J. Julian, S.~Karaman, and D.~Rus, ``On mutual information-based control of range sensing robots for mapping applications,'' \emph{The International Journal of Robotics Research}, vol.~33, no.~10, pp. 1375--1392, 2014.

\bibitem{guo2023patroller}
H.~Guo, Q.~Kang, W.-Y. Yau, M.~H. Ang, and D.~Rus, ``Em-patroller: Entropy maximized multi-robot patrolling with steady state distribution approximation,'' \emph{IEEE Robotics and Automation Letters}, vol.~8, no.~9, pp. 5712--5719, 2023.

\bibitem{duan2026maximizing}
X.~Duan, W.~Wang, and R.~Yan, ``Maximizing markov trajectory entropy under kemeny constraints for robotic surveillance,'' \emph{IEEE Transactions on Automatic Control}, 2026.

\bibitem{lavalle1993game}
S.~M. LaValle and S.~Hutchinson, ``Game theory as a unifying structure for a variety of robot tasks,'' in \emph{Proceedings of 8th IEEE international symposium on intelligent control}.\hskip 1em plus 0.5em minus 0.4em\relax IEEE, 1993, pp. 429--434.

\bibitem{huo2023task}
X.~Huo, H.~Zhang, C.~Huang, Z.~Wang, and H.~Yan, ``Task allocation with minimum requirements for multiple mobile robot systems: A game-theoretical approach,'' \emph{IEEE Transactions on Network Science and Engineering}, vol.~11, no.~1, pp. 1202--1213, 2023.

\bibitem{emery2005game}
R.~Emery-Montemerlo, G.~Gordon, J.~Schneider, and S.~Thrun, ``Game theoretic control for robot teams,'' in \emph{Proceedings of the 2005 IEEE International Conference on Robotics and Automation}.\hskip 1em plus 0.5em minus 0.4em\relax IEEE, 2005, pp. 1163--1169.

\bibitem{roldan2018CompeteOrCooperate}
J.~Jesús~Roldán, J.~Del~Cerro, and A.~Barrientos, ``Should we compete or should we cooperate? {A}pplying game theory to task allocation in drone swarms,'' in \emph{2018 IEEE/RSJ International Conference on Intelligent Robots and Systems (IROS)}, 2018, pp. 5366--5371.

\bibitem{xiao2024task}
H.~Xiao, Z.~Huang, Z.~Xu, S.~Yang, W.~Wang, L.~Zhong, and C.~Xu, ``Task-driven cooperative internet of robotic things crowdsourcing: From the perspective of hierarchical game theoretic,'' \emph{IEEE Internet of Things Journal}, vol.~11, no.~20, pp. 32\,350--32\,362, 2024.

\bibitem{karam2025resource}
R.~Karam, R.~Lin, B.~A. Butler, and M.~Egerstedt, ``Resource allocation with multi-team collaboration based on {H}amilton’s rule,'' in \emph{Proceedings of the IEEE Conference on Decision and Control (CDC)}.\hskip 1em plus 0.5em minus 0.4em\relax IEEE, 2025, pp. 6891--6898.

\bibitem{butlerCDC2025hamiltonsrule}
B.~A. Butler and M.~Egerstedt, ``{H}amilton's rule for enabling altruism in multi-agent systems,'' in \emph{2025 IEEE 64th Conference on Decision and Control (CDC)}, 2025, pp. 6776--6783.

\bibitem{deng1994ComplexityCoop}
X.~Deng and C.~H. Papadimitriou, ``On the {{Complexity}} of {{Cooperative Solution Concepts}},'' \emph{Mathematics of Operations Research}, vol.~19, no.~2, pp. 257--266, May 1994.

\bibitem{stirling2005social1}
W.~C. Stirling, ``Social utility functions-{Part I}: Theory,'' \emph{IEEE Transactions on Systems, Man, and Cybernetics, Part C (Applications and Reviews)}, vol.~35, no.~4, pp. 522--532, 2005.

\bibitem{stirling2005social2}
W.~C. Stirling and R.~L. Frost, ``Social utility functions-{Part II}: Applications,'' \emph{IEEE Transactions on Systems, Man, and Cybernetics, Part C (Applications and Reviews)}, vol.~35, no.~4, pp. 533--543, 2005.

\bibitem{bester1998altruism}
H.~Bester and W.~G{\"u}th, ``Is altruism evolutionarily stable?'' \emph{Journal of Economic Behavior \& Organization}, vol.~34, no.~2, pp. 193--209, 1998.

\bibitem{de2011altruistic}
G.~De~Marco and J.~Morgan, ``Altruistic behavior and correlated equilibrium selection,'' \emph{International Game Theory Review}, vol.~13, no.~04, pp. 363--381, 2011.

\bibitem{tobias2023rational}
{\'A}.~T{\'o}bi{\'a}s, ``Rational altruism,'' \emph{Journal of Economic Behavior \& Organization}, vol. 207, pp. 50--80, 2023.

\bibitem{ale2013evolution}
S.~B. Ale, J.~S. Brown, and A.~T. Sullivan, ``Evolution of cooperation: combining kin selection and reciprocal altruism into matrix games with social dilemmas,'' \emph{PloS one}, vol.~8, no.~5, p. e63761, 2013.

\bibitem{su2024relational}
R.~Su and B.~Morsky, ``Relational utility and social norms in games,'' \emph{Mathematical Social Sciences}, vol. 127, pp. 54--61, 2024.

\bibitem{bjorke1996framework}
J.~T. Bj{\o}rke, ``Framework for entropy-based map evaluation,'' \emph{Cartography and Geographic Information Systems}, vol.~23, no.~2, pp. 78--95, 1996.

\bibitem{o2012gaussian}
S.~T. O’Callaghan and F.~T. Ramos, ``Gaussian process occupancy maps,'' \emph{The International Journal of Robotics Research}, vol.~31, no.~1, pp. 42--62, 2012.

\bibitem{nagami2026vista}
K.~Nagami, T.~Chen, J.~Yu, O.~Shorinwa, M.~Adang, C.~Dougherty, E.~Cristofalo, and M.~Schwager, ``Vista: Open-vocabulary, task-relevant robot exploration with online semantic gaussian splatting,'' \emph{IEEE Robotics and Automation Letters}, 2026.

\end{thebibliography}

\begin{IEEEbiography}[{\includegraphics[width=1in,height=1.25in,
clip,keepaspectratio]{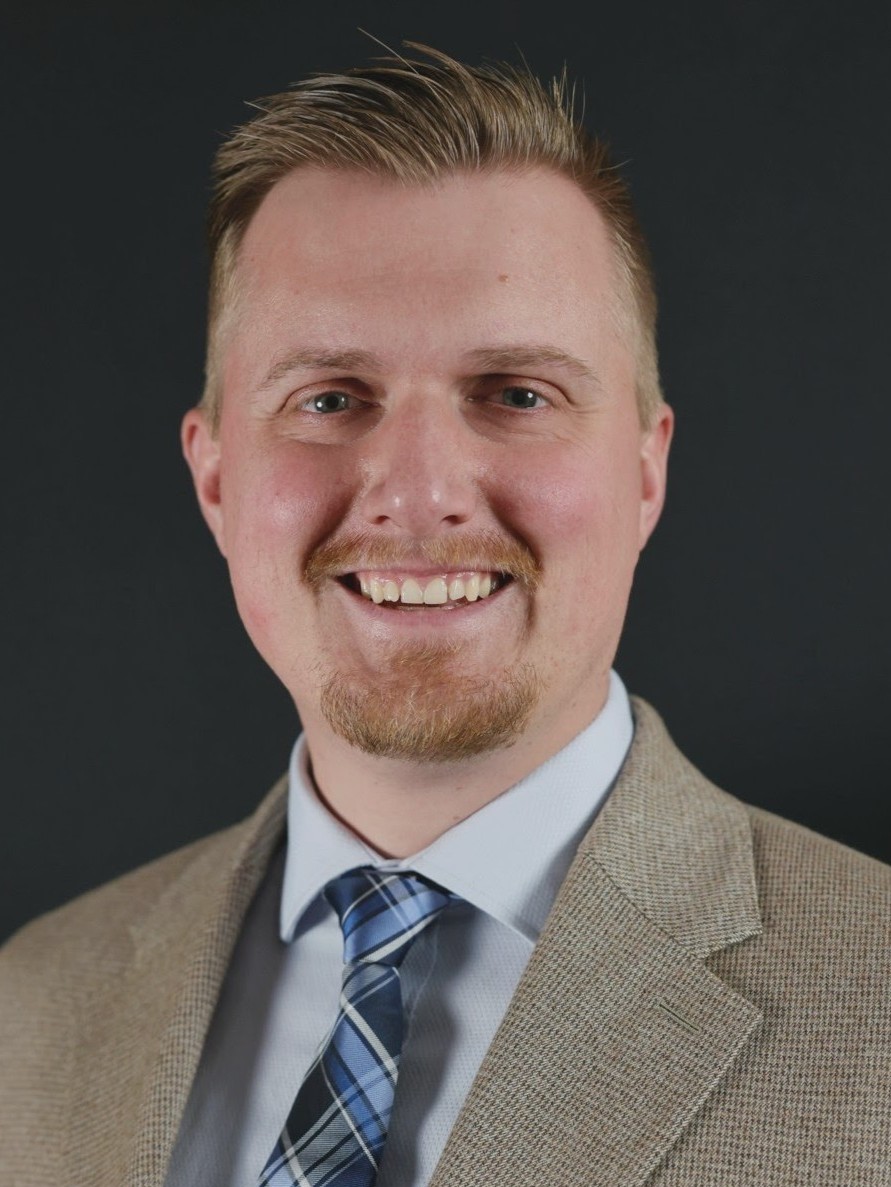}}]{Brooks A. Butler} (Member, IEEE) is an Assistant Professor in the College of Electrical and Computer Engineering at Oklahoma State University in Stillwater, OK. He was the recipient of the Intelligence Community Postdoctoral Research Fellowship, administered by Oak Ridge Institute for Science and Education (ORISE) through an interagency agreement between the U.S. Department of Energy and the Office of the Director of National Intelligence (ODNI), hosted at the University of California, Irvine. He received his Ph.D. in Electrical and Computer Engineering from Purdue University in 2024. He received his M.S. in Computer Science and his B.S. in Applied Physics from Brigham Young University in 2020 and 2019, respectively.
\end{IEEEbiography}

\begin{IEEEbiography}[{\includegraphics[width=1in,height=1.25in,
clip,keepaspectratio]{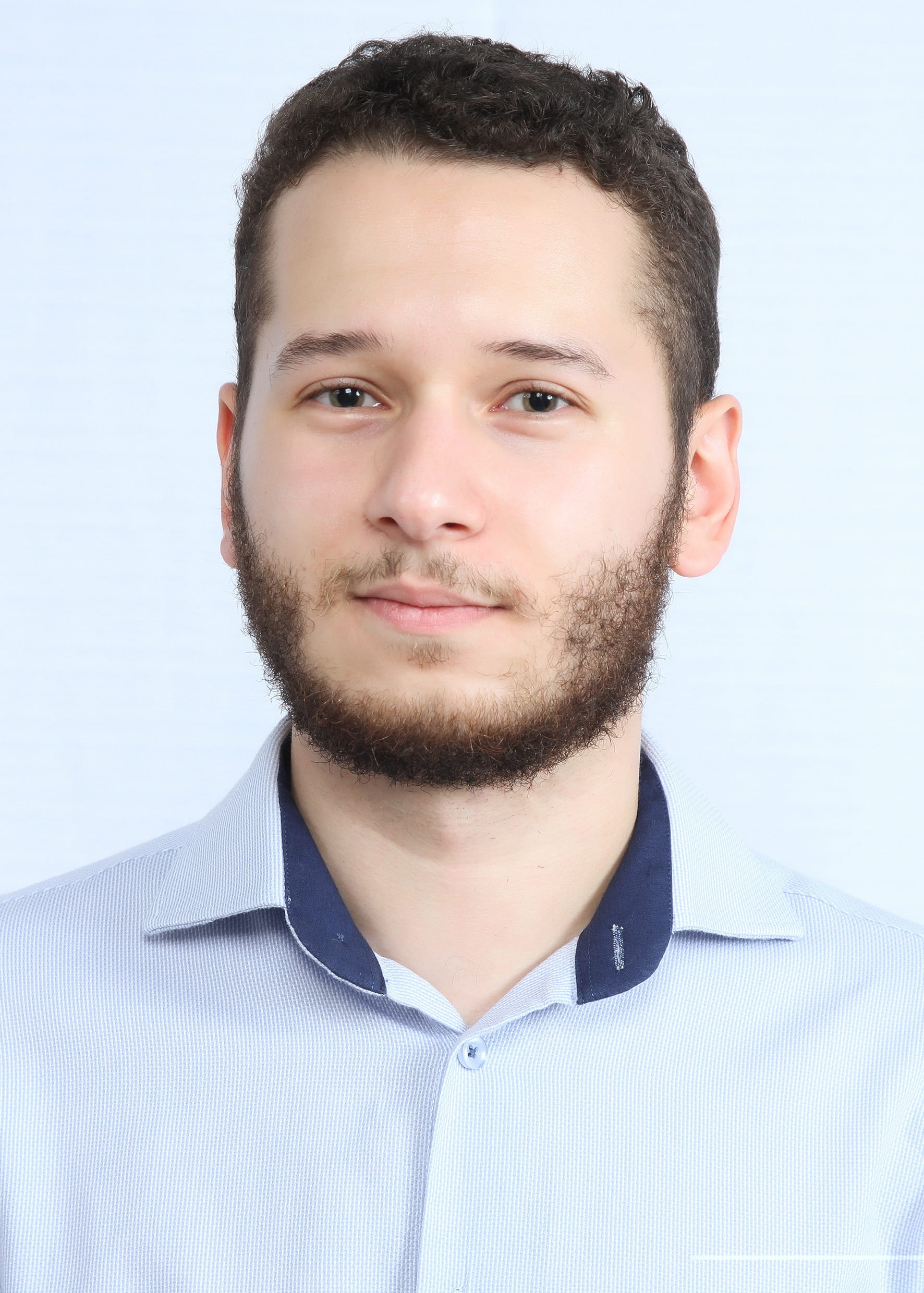}}]{Jair Cert\'orio} (Member, IEEE) is an Assistant Professor in the Systems Engineering Division at Instituto Tecnol\'ogico de Aeron\'autica, Fortaleza, CE, Brazil. 
He received his B.S. degree in Electrical Engineering from the Universidade Estadual de Maringá, PR, Brazil, in 2018, and his Ph.D. degree in Electrical Engineering from the University of Maryland, College Park, MD, USA, in 2025.
Between 2025 and 2026, he was a Postdoctoral Scholar in the Department of Electrical and Computer Engineering at the University of California, Santa Barbara.
\end{IEEEbiography}

\begin{IEEEbiography}[{\includegraphics[width=1in,height=1.25in,
clip,keepaspectratio]{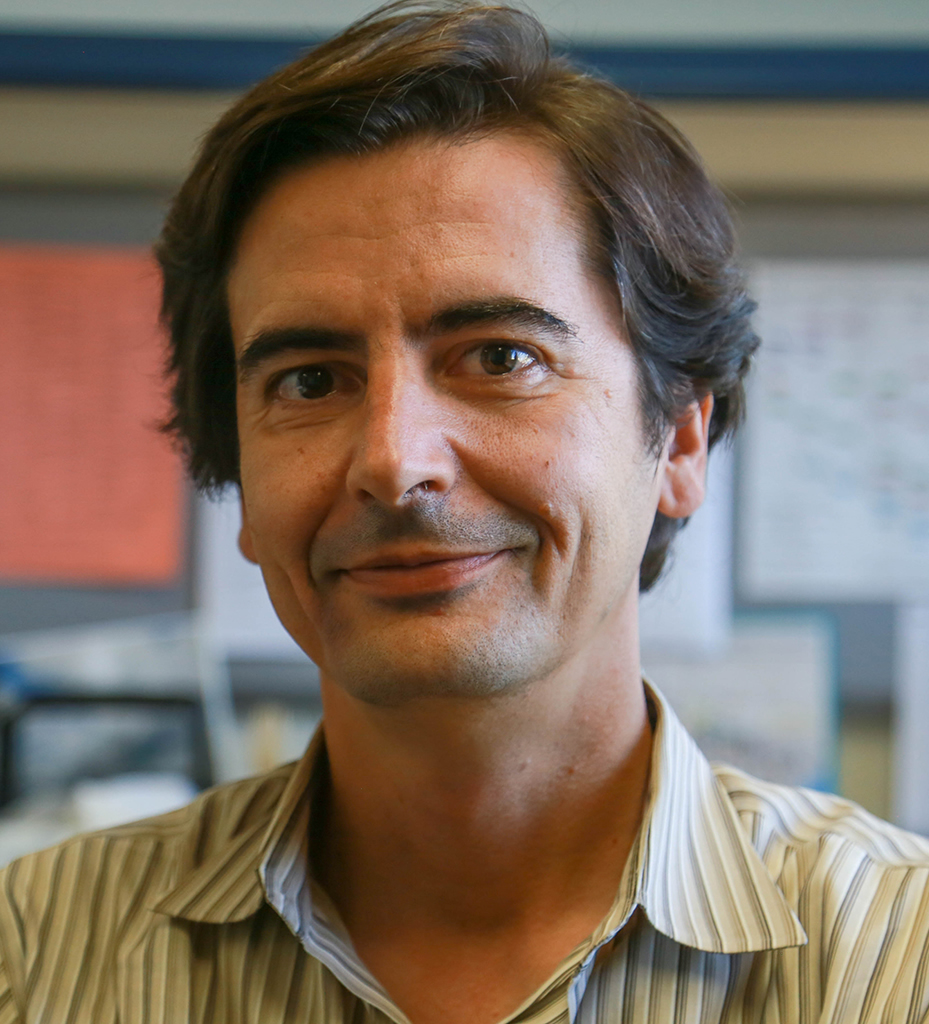}}]{Jo\~ao Hespanha} (Fellow, IEEE) is a Distinguished Professor in the Department of Electrical and Computer Engineering at 
the University of California, Santa Barbara. He received the Ph.D.~degree in electrical engineering and applied science from Yale University, New Haven, Connecticut in 1998. Dr. Hespanha is a Fellow of IEEE and IFAC. His current research interests include hybrid and switched systems; multi-agent control systems; game theory; optimization; distributed control over communication networks (also known as networked control systems); the use of vision in feedback control; stochastic modeling in biology; and network security.
\end{IEEEbiography}


\begin{IEEEbiography}[{\includegraphics[width=1in,height=1.25in,
clip,keepaspectratio]{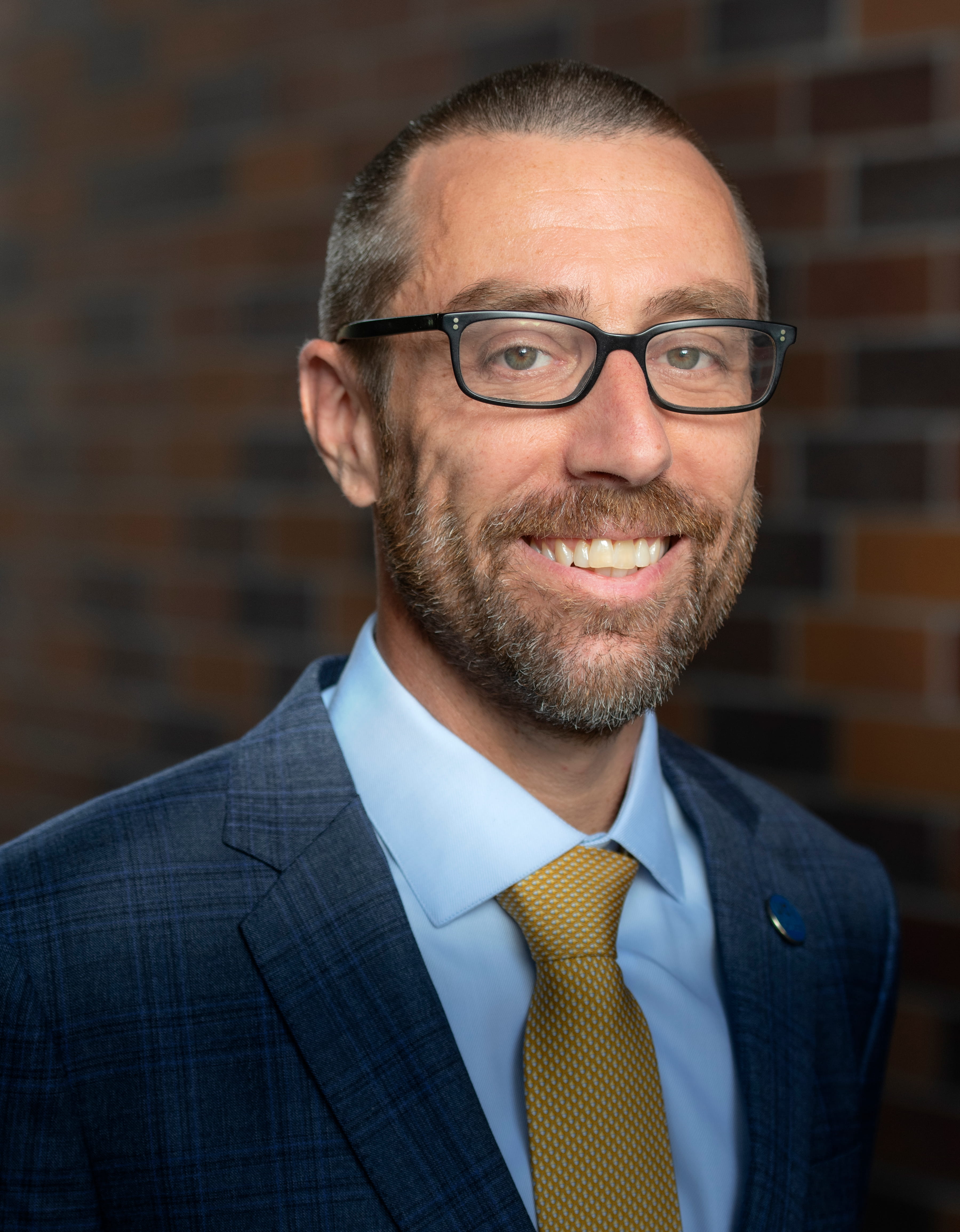}}]{Magnus Egerstedt} (Fellow, IEEE)
is the Provost at the University of North Carolina, Chapel Hill, where he holds the Harry L. Watson Distinguished Professorship in Computer Science. Egerstedt received the M.S. degree in Engineering Physics and the Ph.D. degree in Applied Mathematics from the Royal Institute of Technology, Stockholm, Sweden, the B.A. degree in Philosophy from Stockholm University, and he conducts research in the areas of control theory and robotics, with particular focus on control and coordination of multi-robot systems. Magnus Egerstedt is a Fellow of IEEE and IFAC, a member of the Royal Swedish Academy of Engineering Science, and a past president of the IEEE Control Systems Society.
\end{IEEEbiography}

\end{document}